\documentclass[11pt]{article}

\usepackage{amsmath, amssymb, graphicx, url}
\usepackage[round]{natbib}
\usepackage[algo2e,ruled]{algorithm2e}
\usepackage[colorlinks,linkcolor=red,anchorcolor=blue,citecolor=blue]{hyperref}
\usepackage{times}
\usepackage{mathtools} 
\usepackage{thmtools, thm-restate}
\usepackage{fullpage}
\usepackage{xspace}
\usepackage{amsmath,amsfonts,bm}

\def\eqref#1{equation~\ref{#1}}

\def\1{\bm{1}}

\DeclareMathAlphabet{\mathsfit}{\encodingdefault}{\sfdefault}{m}{sl}
\SetMathAlphabet{\mathsfit}{bold}{\encodingdefault}{\sfdefault}{bx}{n}

\newcommand{\E}{\mathbb{E}}

\DeclareMathOperator*{\argmax}{arg\,max}
\DeclareMathOperator*{\argmin}{arg\,min}

\DeclareMathOperator*{\nth}{^{\text{th}}}
\newcommand\norm[1]{\left\lVert#1\right\rVert}

\providecommand{\P}{}
\renewcommand{\P}{\mathbb{P}}
\newtheorem{theorem}{Theorem}
\newtheorem{remark}[theorem]{Remark}
\newtheorem{definition}[theorem]{Definition}
\newtheorem{assumption}[theorem]{Assumption}
\newtheorem{mylemma}[theorem]{Lemma}

\newtheorem{mycorollary}[theorem]{Corollary}
\providecommand{\proofname}{Proof}
  \newenvironment{proof}%
  {%
   \par\noindent{\bfseries\upshape \proofname\ }%
  }%

\newcommand{\I}[1]{\mathbb{I}\{#1\}}
\DeclarePairedDelimiterX{\infdivx}[2]{(}{)}{%
  #1\;\delimsize\|\;#2%
}
\newcommand{\infdiv}{KL\infdivx}
\DeclareMathOperator*{\pipes}{\|}
\DeclareMathOperator{\polylog}{polylog}

\newcommand{\budgetfix}[1]{\textcolor{red}{#1}}
\renewcommand{\budgetfix}[1]{{#1}}

\usepackage{tikz} %
\usetikzlibrary{arrows.meta,decorations.pathmorphing,backgrounds,positioning,fit,petri}
\usepackage{wrapfig}
\usepackage{enumitem}
\usepackage{authblk}

\begin{document}

\title{Adversarial Online Multi-Task Reinforcement Learning}
\author{Quan Nguyen}
\author{Nishant A. Mehta}
\affil{Department of Computer Science, University of Victoria \thanks{Emails: manhquan233@gmail.com, nmehta@uvic.ca}}

\date{}
\maketitle
\begin{abstract}%
  We consider the adversarial online multi-task reinforcement learning setting, where in each of $K$ episodes the learner is given an unknown task taken from a finite set of $M$ unknown finite-horizon MDP models. 
The learner's objective is to minimize its regret with respect to the optimal policy for each task. We assume the MDPs in $\mathcal{M}$ are well-separated under a notion of $\lambda$-separability, and show that this notion generalizes many task-separability notions from previous works. We prove a minimax lower bound of $\Omega(K\sqrt{DSAH})$ on the regret of any learning algorithm and an instance-specific lower bound of $\Omega(\frac{K}{\lambda^2})$ in sample complexity for a class of \emph{uniformly good} cluster-then-learn algorithms. We use a novel construction called $\emph{2-JAO MDP}$ for proving the instance-specific lower bound. The lower bounds are complemented with a polynomial time algorithm that obtains $\tilde{O}(\frac{K}{\lambda^2})$ sample complexity guarantee for the clustering phase and $\tilde{O}(\sqrt{MK})$ regret guarantee for the learning phase, indicating that the dependency on $K$ and $\frac{1}{\lambda^2}$ is tight.
\end{abstract}

\section{Introduction}
\label{sec:intro}
The majority of theoretical works in online reinforcement learning (RL) have focused on single-task settings in which the learner is given the same task in every episode.
In practice, an autonomous agent might face a sequence of different tasks. For example, an automatic medical diagnosis system could be given an arbitrarily ordered sequence of patients who are suffering from an unknown set of variants of a virus. In this example, the system needs to classify and learn the appropriate treatment for each variant of the virus. This example is an instance of the adversarial online multi-task episodic RL setting, an important learning setting for which the theoretical understanding is rather limited.
The framework commonly used in existing theoretical works is an episodic setting of $K$ episodes; in each episode an unknown Markov decision process (MDP) from a finite set $\mathcal{M}$ of size $M$ is given to the learner. When $M=1$, the setting reduces to single-task episodic RL. 
Most existing algorithms for single-task episodic RL are based on aggregating samples in all episodes to obtain sub-linear bounds on various notions of regret~\citep{Azar2017, ChiJin2018,Simchowitz2019} or finite $(\epsilon, \delta)$-PAC bounds on the sample complexity of exploration~\citep{Dann2015}. 
When $M > 1$, without any assumptions on the common structure of the tasks, aggregating samples from different tasks could produce negative transfer~\citep{BrunskillAndLi2013}.
To avoid negative transfer, existing works~\citep{BrunskillAndLi2013, Hallak2015contextual, LatentMDPs2021} assumed that there exists some notion of task-separability that defines how different the tasks in $\mathcal{M}$ are. Based on this notion of separability, most existing algorithms followed a two-phase cluster-then-learn paradigm that first attempts to figure out which MDP is being given and then uses the samples from the previous episodes of the same MDP for learning. 
However, most existing works employ strong assumptions such that the tasks are given stochastically following a fixed distribution~\citep{Azar2013,BrunskillAndLi2013,Steimle2021,LatentMDPs2021} or the task-separability notion allows the MDPs to be distinguished in a small number of exploration steps~\citep{Hallak2015contextual,LatentMDPs2021}.
These strong assumptions become the main theoretical challenges towards understanding this setting.

Our goal in this work is to study the adversarial setting with a more general task-separability notion, in which the aformentioned strong assumptions do not hold. Specifically, the learner makes no statistical assumptions on the sequence of tasks; the task in each episode can be either the same or
different from the tasks in any other episodes. 
Moreover, the difference between the tasks in two consecutive episodes can be large (linear in the length of the episodes) so that algorithms based on a fixed budget for total variation such as RestartQ-UCB~\citep{Mao2021b} cannot be applied. 
The performance of the learner is measured by its regret with respect to an omniscient agent that knows which tasks are coming in every episode and the optimal policies for these tasks.
We consider the same cluster-then-learn paradigm of the previous works and focus on the following two questions:

\begin{center}
  \begin{itemize}[leftmargin=6mm]
    \item \textit{Is there a task-separability notion that generalizes the notions from previous works while still enabling tasks to be distinguished by a cluster-then-learn algorithm with polynomial time and sample complexity? If so, what is the optimal sample complexity of clustering under this notion?}
    \item \textit{Is there a polynomial time cluster-then-learn algorithm that simultaneously obtains near-optimal sample complexity in the clustering phase and near-optimal regret guarantee for the learning phase in the adversarial setting?}
  \end{itemize}    
\end{center}

We answer both questions positively. For the first question, we introduce the notion of $\lambda$-separability, a task-separability notion that generalizes the task-separability definitions in previous works in the same setting~\citep{BrunskillAndLi2013, Hallak2015contextual, LatentMDPs2021}. 
Definition~\ref{def:lambdaSeparability} formally defines $\lambda$-separability. 
A more informal version of $\lambda$-separability has appeared in the discounted setting of Concurrent PAC RL~\citep{GuoAndBrunskill2015} where multiple MDPs are learned concurrently; however the implications on the episodic sequential setting and the tightness of their results were lacking. In essence, $\lambda$-separability assumes that between every pair of MDPs in $\mathcal{M}$, there exists some state-action pair whose transition functions are well-separated in $\ell_1$-norm. This setting is more challenging than the one considered by~\citet{Hallak2015contextual} where \textit{all} state-action pairs are well-separated. 
In Appendix~\ref{appendix:Example}, we show that $\lambda$-separability is more general than the entropy-based separability defined in~\citet{LatentMDPs2021} and thus requires novel approaches to exploring and clustering samples from different episodes. 
Under this notion of $\lambda$-separability, we show an instance-specific lower bound\footnote{Here and throughout the introduction, we suppress factors related to the MDPs such that the number of states and actions and the horizon length in all the bounds.} $\Omega(\frac{K}{\lambda^2})$ on both the sample complexity and regret of the clustering phase for a class of cluster-then-learn algorithms that includes most of the existing works.

To answer the second question, we propose a new cluster-then-learn algorithm, AOMultiRL, which obtains a regret upper bound of $\tilde{O}\left(\frac{K}{\lambda^2} + \sqrt{MK}\right)$ (the $\tilde{O}$ hides logarithmic terms). This upper bound indicates that the linear dependency on $K$ and $\lambda^2$ in the lower bounds are tight.  
The $\tilde{O}(\sqrt{MK})$ upper bound in the learning phase is near-optimal because if the identity of the model is revealed to a learner at the beginning of every episode (so that no clustering is necessary), there exists a straightforward $\Omega(\sqrt{MK})$ lower bound obtained by combining the lower bound for the single-task episodic setting of~\citet{Domingues2021} and Cauchy-Schwarz inequality. In the stochastic setting, the L-UCRL algorithm~\citep{LatentMDPs2021} obtains $O(\sqrt{MK})$ regret with respect to the optimal policy of a partially observable MDP (POMDP) setting that does not know the identity of the MDPs in each episode; thus their notion of regret is weaker than the one in our work.

 \subsection*{Overview of Techniques}
 \begin{itemize}[leftmargin=6mm]
  \item In Section~\ref{sec:lowerbounds}, we present two lower bounds. The first is a minimax lower bound $\Omega(K\sqrt{SAH})$ on the total regret of any algorithm. This result uses the construction of JAO MDPs in~\citet{UCRL2}. The second is a $\Omega\left(\frac{K}{\lambda^2}\right)$ instance-specific lower bound on the sample complexity and regret of the clustering phase for a class of \textit{uniformly good} cluster-then-learn algorithms when both $\lambda$ and $M$ are sufficiently large. The instance-specific lower bound relies on the novel construction of \emph{2-JAO MDP}, a hard instance combining two JAO MDPs in which one is the minimax lower bound instance and the other satisfies $\lambda$-separability. We show that learning 2-JAO MDPs is fundamentally a two-dimensional extension of the problem of finding a biased coin among a collection of fair coins~\citep[e.g.][]{Tulsiani2014L6}, for which information theoretic techniques of the one-dimensional problem can be adapted.
  
  \item In Section~\ref{sec:upperbound}, we show that AOMultiRL obtains a regret upper bound of $\tilde{O}\left(\frac{K}{\lambda^2} + \sqrt{MK}\right)$. 
  The main idea of AOMultiRL is based on the observation that a fixed horizon of order $\Theta(\frac{1}{\lambda^2})$ with a small constant factor is sufficient to obtain a $\lambda$-dependent coarse estimate of the transition functions of all state-action pairs. In turn, this coarse estimate is sufficient to have high-probability guarantees for the correctness of the clustering phase. This allows AOMultiRL to have a fixed horizon for the learning phase and be able to apply single-task RL algorithms with theoretical guarantees such as UCBVI-CH~\citep{Azar2017} in the learning phase.
\end{itemize}

Our paper is structured as follows: Section~\ref{sec:setup} formally sets up the problem.
 Section~\ref{sec:lowerbounds} presents the lower bounds. AOMultiRL and its regret upper bound are shown in Section~\ref{sec:upperbound}. 
Several numerical simulations are in Section~\ref{sec:experiments}. 
The appendix contains formal proofs of all results. 
We defer detailed discussion on related works to Appendix~\ref{appendix:RelatedWorks}.

\section{Problem Setup}
\label{sec:setup}
Our learning setting consists of $K$ episodes. 
In episode $k = 1, 2, \dots, K$, an adversary chooses an unknown Markov decision process (MDP) $m^k$ from a set of finite-horizon tabular stationary MDP models $\mathcal{M} = \{(\mathcal{S}, \mathcal{A}, H, P_i, r): i = 1, 2, \dots, M\}$ where $r: \mathcal{S \times A} \mapsto [0,1]$ is the shared reward function, $\mathcal{S}$ is the set of states with size $S$, $\mathcal{A}$ is the set of actions with size $A$, $H$ is the length of each episode, and $P_i: \mathcal{S \times A \times S} \mapsto [0,1]$ is the transition function where $P_i(s'|s,a)$ specifies the probability of being in state $s'$ after taking action $a$ at state $s$.  
The state space $\mathcal{S}$ and action space $\mathcal{A}$ are known and shared between all models; however, the transition functions are distinct and unknown. 
Following a common practice in single-task RL literature~\citep{Azar2017,ChiJin2018}, we assume that the reward function is known and deterministic, however our techniques and results extend to the setting of unknown stochastic $r$. 
Furthermore, the MDPs are assumed to be communicating with a finite diameter $D$~\citep{UCRL2}. A justification for this assumption on the diameter is provided in Section~\ref{sec:diameter}.

The adversary also chooses the initial state $s^k_1$. The policy $\pi_k$ of the learner in episode $k$ is a collection of $H$ functions $\pi^k = \{\pi_{k,h}: \mathcal{S} \mapsto \mathcal{A}\}$, which can be non-stationary and history-dependent. The value function of $\pi_k$ starting in state $s$ at step $h$ is the expected rewards obtained by following $\pi_k$ for $H-h+1$ steps $V_h^{\pi_k}(s) = \E[\sum_{h'=h}^H r(s^k_{h'}, \pi^k_h(s^k_{h'})) \mid s^k_h = s]$, where the expectation is taken with respect to the stochasticity in $m^k$ and $\pi^k$. Let $V_1^{k,*}$ denote the value function of the optimal policy in episode $k$.

The performance of the learner is measured by its regret with respect to the optimal policies in every episode:
\begin{equation}
\label{eq:Regret}
\text{Regret}(K) = \sum_{k=1}^K[V^{k,*}_1 - V^{\pi_k}_1](s^k_1).
\end{equation}
Let $[M] = \{1, 2, \dots, M\}$. We assume that the MDPs in $\mathcal{M}$ are \emph{$\lambda$-separable}: 

\begin{definition}[$\lambda$-separability]
    Let $\lambda > 0$ and consider set of MDP models $\mathcal{M} =  \{m_1, \dots, m_M\}$ with $M$ models. For all $(i, j) \in [M] \times [M]$ and $i \neq j$, the $\lambda$-distinguishing set for two models $m_i$ and $m_j$ is defined as the set of state-action pairs such that the $\ell_1$ distance between $P_i(s,a)$ and $P_j(s,a)$ is larger than  $\lambda$:
        $\Gamma^{\lambda}_{i,j} = \{(s,a) \in \mathcal{S \times A}: \norm{P_{i}(s,a) - P_{j}(s,a)} \geq \lambda\}$, 
    where $\norm{\cdot}$ denotes the $\ell_1$-norm and $P_i(s,a) = P_i(\cdot \mid s, a)$.

    The set $\mathcal{M}$ is $\lambda$-separable if for every two models $m_i, m_j$ in $\mathcal{M}$, the set $\Gamma^{\lambda}_{i,j}$ is non-empty:
    \begin{equation*}
    \forall i,j \in [M], i \neq j: \Gamma^{\lambda}_{i,j} \neq \emptyset.
    \end{equation*}
    In addition, $\lambda$ is called a separation level of $\mathcal{M}$, and we say a state-action pair $(s,a)$ is $\lambda$-distinguishing for two models $m_i$ and $m_j$ if $\norm{P_{i}(s,a) - P_{j}(s,a)} > \lambda$. 
    \label{def:lambdaSeparability}
\end{definition}

We use the following notion of a $\lambda$-distinguishing set for a collection of MDP models $\mathcal{M}$:
\begin{definition}[$\lambda$-distinguishing set]
  Given a $\lambda$-separable set of MDPs $\mathcal{M}$, a $\lambda$-distinguishing set of $\mathcal{M}$ is a set of state-action pairs $\Gamma^\lambda \subseteq \mathcal{S} \times \mathcal{A}$ such that for all
    $i, j \in [M], \Gamma^{\lambda}_{i, j} \cap \Gamma^\lambda \neq \emptyset$.
  In particular, the set $\Gamma = \cup_{i,j}\Gamma^{\lambda}_{i,j}$ is a $\lambda$-distinguishing set of $\mathcal{M}$.
\end{definition}
By definition, a state-action pair can be $\lambda$-distinguishing for some pairs of models and not $\lambda$-distinguishing for other pairs of models. 

\subsection{Assumption on the finite diameter of the MDPs}
\label{sec:diameter}
In this work, all MDPs are assumed to be communicating. We employ the following formal definition and assumption commonly used in 
literature~\citep{UCRL2, BrunskillAndLi2013,Sun2020, Tarbouriech2021a}:

\begin{definition}(\citep{UCRL2}) 
Given an ergodic Markov chain $\mathcal{F}$, let $T^\mathcal{F}_{s, s'} = \inf\{t > 0 \mid s_t = s', s_0 = s\}$ be the first passage time for two states $s, s'$ on $\mathcal{F}$. 
  Then the hitting time of a unichain MDP $G$ is $T_G = \max_{s,s' \in S}\max_{\pi}\mathbb{E}[T^{\mathcal{F}_\pi}_{s, s'}]$, where $\mathcal{F}_\pi$ is the Markov chain induced by $\pi$ on $G$. In addition, $T'_G = \max_{s, s' \in S}\min_{\pi}\mathbb{E}[T^{\mathcal{F}_\pi}_{s, s'}]$ is the diameter of $G$.
\end{definition}

\begin{assumption} The diameter of all MDPs in $\mathcal{M}$ are bounded by a constant $D$.
  \label{assumption:diameter}
  \end{assumption}

While this finite diameter assumption is common in undiscounted and discounted single-task setting~\citep{UCRL2,GuoAndBrunskill2015}, it is not necessary in the episodic single-task setting~\citep{ChiJin2018, Mao2021b}. 
Therefore, it is important to justify this assumption in the episodic multi-task setting. 
In the episodic single-task setting, for any initial state $s_1$, the average time between any pair of states reachable from $s_1$ is bounded $2H$; hence, $H$ plays the same role as $D$~\citep{Domingues2021}.
This allows the learner to visit and gather state-transition samples in each state multiple times and construct accurate estimates of the model.  

However, in the multi-task setting, the same initial state $s_1$ in one episode might belong to a different MDP than the state $s_1$ in the previous episodes. Therefore, the set of reachable states and their state-transition distributions could change drastically.
Hence, it is important that the $\lambda$-distinguishing state-action pairs be reachable from any initial state $s_1$ for the learner to recognize which MDP it is in and use the samples appropriately. 
Otherwise, combining samples from different MDPs could lead to negative transfer.
Conversely, if the MDPs are allowed to be non-communicating, the component that makes them $\lambda$-separable might be unreachable from other components. In this case, the adversary can pick the initial states in these components and block the learner from accessing the $\lambda$-distinguishing state-actions. A construction that formalizes this argument is shown at the end of Section~\ref{sec:lowerbounds}.

\section{Minimax and Instance-Dependent Lower Bounds}
\label{sec:lowerbounds}

We first show that if $\lambda$ is sufficiently small and $M = \Theta(SA)$, then the setting is uninteresting in the sense that one cannot do much better than learning every episode individually without any transfer, leading to an expected regret that grows linearly in the number of episodes $K$.

\begin{restatable}[Minimax Lower Bound]{mylemma}{LemmaMinimaxLowerBound}
  Suppose $S, A \geq 10, D \geq 20\log_A(S)$ and $H \geq DSA$ are given. Let $\lambda = \Theta(\sqrt{\frac{SA}{HD}})$. There exists a set of $\lambda$-separable MDPs $\mathcal{M}$ of size $M = \frac{SA}{4}$, each with $S$ states, $A$ actions, diameter at  most $D$ and horizon $H$ such that if the tasks are chosen uniformly at random from $\mathcal{M}$, the expected regret of any sequence of policies $(\pi_k)_{k=1,\dots,K}$ over $K$ episodes is
  \begin{align*}
    \E[\mathrm{Regret}(K)] \geq \Omega\left(K\sqrt{DSAH}\right).
  \end{align*}
  \label{lemma:minimaxLowerBound}
\end{restatable}
\vspace{-.8cm}
\begin{proof}{(Sketch)}
  We construct $\mathcal{M}$ so that each MDP in $\mathcal{M}$ is a JAO MDP~\citep{UCRL2} of two states $\{0, 1\}$, $\frac{SA}{4}$ actions and diameter $\frac{D}{4}$. Figure~\ref{fig:jaomdp} (left) illustrates the structure of a JAO MDP. State $0$ has no reward, while state $1$ has reward $+1$. Each model has a unique best action $a^*$ that starts from $0$ and goes to $1$. The pair $(0, a^*)$ is a $\lambda$-distinguishing state-action pair.
  
  A JAO MDP can be converted to an MDP with $S$ states, $A$ actions and diameter $D$, and this type of MDP gives the minimax lower bound proof in the undiscounted setting~\citep{UCRL2}.
The adversary selects a model from $\mathcal{M}$ uniformly at random, and so previous episodes provide no useful information for the current episode; hence, the regret of any learner is equal to the sum of its $K$ one-episode learning regrets.
The one-episode learning regret for JAO MDPs is known to be $\Omega(\sqrt{D S A H})$ when comparing against the optimal infinite-horizon average reward. For JAO MDPs, the optimal infinite horizon policy is also optimal for finite horizon; so, we can use a geometric convergence result from Markov chain theory~\citep{MarkovChainAndMixingTime2008} to convert this lower bound to a lower bound of the standard finite-horizon regret of the same order, giving the result.
  \end{proof}

  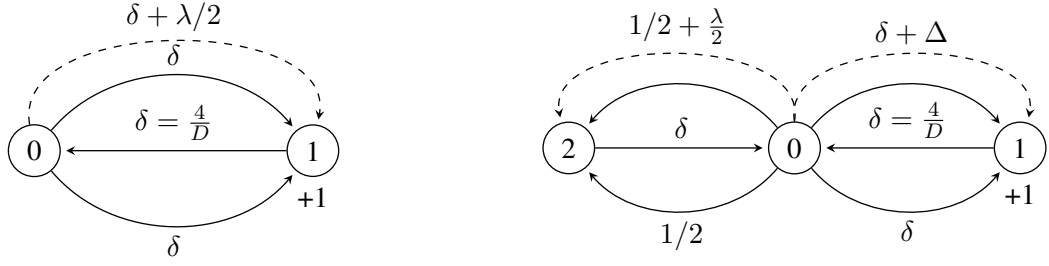
\begin{figure}[tbp]
    \centering
    \begin{minipage}{.5\textwidth}
    \centering
    \begin{tikzpicture}[->,>= stealth,shorten >=2 pt, line width =0.5 pt , node distance = 4 cm]
      \node [circle, draw] (zero) {0};
      \node [label=below:{+1}, circle , draw] (one) [ right = 3cm of zero] {1};
      \path (zero) edge [dashed, bend left = 100] node [above] {$\delta+\lambda/2$} (one) ;
      \path (zero) edge [bend left = 50] node [above] {$\delta$} (one) ;
      \path (zero) edge [ bend right = 50] node [below] {$\delta$} (one) ;
      \path (one) edge [ bend right = 0] node [above] {$\delta = \frac{4}{D}$} (zero);
    \end{tikzpicture}
  \end{minipage}%
  \begin{minipage}{0.5\textwidth}
      \centering
	\begin{tikzpicture}[->,>= stealth,shorten >=2 pt, line width =0.5 pt , node distance = 3 cm]
		\node [circle, draw] (zero) {0};
		\node [label=below:{+1},  circle , draw] (one) [ right of = zero] {1};
		\node [circle, draw] (two) [left of = zero] {2};
		\path (zero.north) edge [dashed, bend left = 100] node [above] {$\delta+\Delta$} (one);
		\path (zero) edge [ bend left = 50] node [above] {} (one) ;
		\path (zero) edge [ bend right = 50] node [below] {$\delta$} (one) ;
		\path (one) 
			edge [ bend right = 0] node [above] {$\delta = \frac{4}{D}$} (zero)
			;
		\path (zero.north) edge [dashed, bend right = 100] node [above] {$1/2+\frac{\lambda}{2}$} (two);
		\path (zero) edge [bend right = 50] node [above] {} (two);
		\path (zero) edge [bend left = 50] node [below] {$1/2$} (two);	
		\path (two) 
			edge [ bend right = 0] node [above] {$\delta$} (zero)
			;
	\end{tikzpicture}
\end{minipage}
    \caption{A JAO MDP (left) and a 2-JAO MDP (right). Only state $1$ has reward $+1$. The dashed arrows indicate the best actions.}
    \label{fig:jaomdp}
    \end{figure}

  Using the same technique in the proof of Lemma~\ref{lemma:minimaxLowerBound}, we can show that applying UCRL2~\citep{UCRL2} in every episode individually leads to a regret upper bound of $O\left(KDS\sqrt{AH\ln{H}}\right)$. This implies that learning every episode individually already gives a near-optimal regret guarantee.

\begin{remark}
  Our proof for Lemma~\ref{lemma:minimaxLowerBound} contains a simple yet rigorous proof for the mixing-time argument used in~\citet{Mao2021b,ChiJin2018}. This argument claims that for JAO MDPs, when the diameter is sufficiently small compared to the horizon, the optimal $H$-step value function $V^*_1$ in the regret of the episodic setting can be replaced by the optimal average reward $\rho^*H$ in the undiscounted setting without changing the order of the lower bound. 
  To the best of our knowledge, our proof is the first rigorous proof for this argument that applies for any number of episodes including $K = 1$.~\citet{Domingues2021} provide an alternative proof; however the results therein hold in a different setting where $K$ is sufficiently large and the horizon $H$ can be much smaller than $D$.
\end{remark}

We emphasize that the lower bound in Lemma~\ref{lemma:minimaxLowerBound} holds for \emph{any} learning algorithms. This result motivates the more interesting setting in which $\lambda$ is a fixed and large constant independent of $H$. In this case, we are interested in an instance-specific lower bound. For multi-armed bandits, instance-specific lower bounds are constructed with respect to a class of \textit{uniformly good} learning algorithms~\citep{LaiAndRobbins1985}. In our setting, we focus on defining a class of uniformly good algorithms that include the cluster-then-learn algorithms in the previous works for multi-task PAC RL settings such as Finite-Model-RL~\citep{BrunskillAndLi2013} and PAC-EXPLORE~\citep{GuoAndBrunskill2015}. We consider a class of MDPs and a cluster-then-learn algorithm uniformly good if they satisfy an intuitive property: for any MDP in that class, the algorithm should be able to correctly classify whether a cluster of samples is from that MDP or not with an arbitrarily low (but not zero) failure probability, provided that the horizon $H$ is sufficiently long for the algorithm to collect enough samples. The following definition formalizes this idea.

\begin{definition}[PAC identifiability of MDPs]
	A set of models $\mathcal{M}$ of size $M$ is PAC identifiable if there exists a function $f: (0,1) \mapsto \mathbb{N}$, a sample collection policy $\pi$ and a classification algorithm $\mathcal{C}$ with the following property: for every $p \in (0, 1)$, for each model $1 \leq m \leq M$ in $\mathcal{M}$, if $\pi$ is run for $f(p)$ steps and the state-transition samples are given to $\mathcal{C}$, then the algorithm $\mathcal{C}$ returns the correct identity of $m$ with probability at least $1-p$, where the probability is taken over all possible sequence of $f(p)$ samples collected by running $\pi$ on $m$ for $f(p)$ steps. The smallest choice of function $f(p)$ among all possible choices is called the sample complexity of model identification of $\mathcal{M}$.
\end{definition}

The clustering algorithm in a cluster-then-learn framework solves a problem different from classification: they only need to tell whether a cluster of samples belong to the same or different distribution than another cluster of samples, not the identity of the distribution. We can reduce one problem to the other by the following construction: consider the adversary that gives all $M$ models in the first $M$ episodes. After the first $M$ episodes, there are $M$ clusters of samples, each corresponding to one model in $\mathcal{M}$. Once the learner has constructed $M$ different clusters, from the episode $M+1$, the clustering problem is as hard as classification since identifying the right cluster immediately implies the identity of the MDP where the samples come from, and vice versa. Hence, we can apply the sample complexity of classification to that of clustering. 

Next, we show the lower bound on the sample complexity of model identification for the class of $\lambda$-separable communicating MDPs. 

\begin{restatable}{mylemma}{LemmaPACLowerBoundQCoins}
  For any $S, A \geq 20, D \geq 16$ and $\lambda \in (0, \frac{1}{2}]$, there exists a PAC identifiable $\lambda$-separable set of MDPs $\mathcal{M}$ of size $\frac{SA}{12}$, each with at most $S$ states, $A$ actions and diameter $D$ such that for any classification algorithm $\mathcal{C}$, if the number of state-transition samples given to $\mathcal{C}$ is less than $\frac{SA}{180\lambda^2}$ then for at least one MDP in $\mathcal{M}$, algorithm $\mathcal{C}$ fails to identify that MDP with probability at least $\frac{1}{2}$.
  \label{lemma:PACLowerBoundQCoins}
\end{restatable}

\begin{proof}{(Sketch)}
The set $\mathcal{M}$ is a set of 2-JAO MDPs, shown in Figure~\ref{fig:jaomdp} (right). Each 2-JAO MDP combines two JAO MDPs with the same number of actions and with diameter in the range $[\frac{D}{2}, D]$; one is $\lambda$-separable and one is the hard instance for the minimax lower bound of~\citet{UCRL2}. 
Rewards exist only in the part containing the hard instance. 
If a learner completely ignores the $\lambda$-separable part, by Lemma~\ref{lemma:minimaxLowerBound} the learner cannot do much better than just learning every episode individually. 
On the other hand, with enough samples from the $\lambda$-separable part, the learner can identify the MDP and use the samples collected in the previous episodes of the same MDP to accelerate learning the hard instance part. 
However, the $\lambda$-separable part is also a JAO MDP, for which no useful information from previous episodes can help identify the MDP in the current episode.

Only the actions at state $0$ are $\lambda$-distinguishing and can be used to identify the MDPs. Taking an action in state $0$ can be seen as flipping a coin: heads for transitioning to another state and tails for staying in state $0$. 
Identifying a 2-JAO MDP reduces to the problem of using at most $H$ coin flips to identify, in a $Q \times 2$ matrix of coins, a row $j$ that has coins that are slightly different from the others. The first column has fair coins except in row $j$, where the success probability is $\frac{1}{2} + \lambda$. The second column coins with success probability of $\delta \leq \frac{1}{4}$ except in row $j$, where the coin is upwardly biased by $\Delta \leq \lambda$. Lemma~\ref{lemma:generalized-tulsiani} and Corollary~\ref{corollary:generalized-tulsiani} in the appendix show a $\Omega\left(\frac{Q}{\lambda^2}\right)$ lower bound on the number of coin flips on the first column (the left part of the 2-JAO MDP), implying the desired result.
\end{proof}

Lemma~\ref{lemma:PACLowerBoundQCoins} imply that for 2-JAO MDPs, any uniformly good model identification algorithm needs to collect at least $\Omega\left(\frac{SA}{\lambda^2}\right)$ samples from state $0$ on the left part. Whenever an action towards state $2$ is taken from state $0$, the learner may end up in state $2$. Once in state $2$, the learner needs to get back to state $0$ to obtain the next useful sample. The expected number of actions needed to get back to state $0$ from state $2$ is $\frac{1}{\delta} = \frac{D}{4}$. This implies the following two lower bounds on the horizon of the clustering phase and the total regret of any cluster-then-learn algorithms.

\begin{restatable}{mycorollary}{CorollaryPACLowerBoundSteps}
  For any $S, A \geq 20, D \geq 16$ and $\lambda \in (0, 1]$, there exists a PAC identifiable $\lambda$-separable set of MDPs $\mathcal{M}$ of size $M = \frac{SA}{12}$, each with $S$ states, $A$ actions and diameter $D$ such that for any uniformly good cluster-then-learn algorithm, to find the correct cluster with probability of at least $\frac{1}{2}$, the expected number of exploration steps needed in the clustering phase is $\Omega(\frac{DSA}{\lambda^2})$. Furthermore, the expected regret over $K$ episodes of the same algorithm is
  \begin{align*}
      \E[\mathrm{Regret}(K)] \geq \Omega\left(\frac{KDSA}{\lambda^2}\right).
  \end{align*}  
  \label{corollary:PACLowerBoundSteps}
\end{restatable}

\begin{proof}(Sketch)
  In the lower bound construction, the learner is assumed to know everything about the set of models, including their optimal policies. Hence, after having identified the model in the clustering phase, the learner can follow the optimal policy in the learning phase and incur a small regret of at most $D/2$ in this phase. Therefore, the regret is dominated by the regret in the clustering phase, which is of order $\frac{DSA}{\lambda^2}$.
\end{proof}  

\begin{remark}
  The lower bound in Corollary~\ref{corollary:PACLowerBoundSteps} holds for a particular class of uniformly good cluster-then-learn algorithms under an adaptive adversary. It remains an open question whether this lower bound holds for any algorithms, not just cluster-then-learn.
\end{remark}

\begin{remark}
  Corollary~\ref{corollary:PACLowerBoundSteps} implies that, without further assumptions, it is not possible to improve the $\frac{1}{\lambda^2}$ dependency on $\lambda$. At the first glance this seems to contradict the existing results in bandits and online learning literature, where the regret bound depends on $\frac{1}{\mathrm{gap}}$ where $\mathrm{gap}$ is the the difference in expected reward between the best arm and the sub-optimal arms. However, $\lambda$ does not play the same role as the gaps in bandits. 
  Observe that on the 2-JAO MDPs, the set of arms with positive reward is only in the right JAO MDP. The lower-bound learner knows this, but chooses to pull the arms on the left JAO MDP (with zero-reward) to collect side information that helps learn the right part faster. In this analogy, $\lambda$ does not play the same role as the gaps in bandits, since the learner already knows the arms on the left JAO MDP are suboptimal. The role of $\lambda$ is in model identification, for which similar $\frac{1}{\lambda^2}$ lower bounds are known~\citep[e.g.][]{Tulsiani2014L6}.
\end{remark}

Finally, we construct a non-communicating variant of the 2-JAO MDP to show that the finite diameter assumption is necessary. Figure~\ref{fig:noncommunicatingtwojaomdp} in Appendix~\ref{sec:proofsLowerBounds} illustrates this construction. On this variant, all the transitions from state $0$ to state $2$ are reversed. In addition, no actions take state $0$ to state $2$, making this MDP non-communicating. A set of these non-communicating MDPs is still $\lambda$-separable due to the state-action pairs that start at state $2$. However, by setting the initial state to $0$, the adversary can force the learner to operate only on the right part, regardless of how large $\lambda$ is.

\section{Non-Asymptotic Upper Bounds}
\label{sec:upperbound}
We propose and analyze AOMultiRL, a polynomial time cluster-then-learn algorithm that obtains a high-probability regret bound of $\tilde{O}(\frac{KDSA}{\lambda^2}+H^{3/2}\sqrt{MSAK})$.
In each episode, the learner starts with the clustering phase to identify the cluster of samples generated in previous episodes that has the same task.
Once the right cluster is identified, the learner can use the samples from previous episodes in the learning phase. 

A fundamental difference between the undiscounted infinite horizon setting considered in previous works~\citep{GuoAndBrunskill2015, BrunskillAndLi2013} and the episodic finite horizon in our work is the horizon of the two phases. In previous works, different episodes might have different horizons for the clustering phase depending on whether the learner decides to start exploration at all~\citep{BrunskillAndLi2015} or which state-action pairs are to be explored~\citep{BrunskillAndLi2013}. 
This poses a challenge for the episodic finite-horizon setting, because a varying horizon for the clustering phase leads to a varying horizon for the learning phase. Thus, standard single-task algorithms that rely on a fixed horizon such as UCBVI~\citep{Azar2017} and StrongEuler~\citep{Simchowitz2019} cannot be applied directly. 
From an algorithmic standpoint, for a fixed horizon $H$, a non-asymptotic bound on the horizon of the clustering phase is necessary so that the learner knows exactly whether $H$ is large enough and when to stop collecting samples. 

AOMultiRL alleviates this issue by setting a fixed horizon for the clustering phase, which reduces the learning phase to standard single-task episodic RL. First, we state an assumption on the ergodicity of the MDPs.

\begin{assumption} 
The hitting times of all MDPs in $\mathcal{M}$ are bounded by a known constant $\tilde{D}$.\label{assumption:hittingtime}
\end{assumption}

The main purpose of Assumption~\ref{assumption:hittingtime} is simplifying the computation of a non-asymptotic upper bound for the clustering phase in order to focus the exposition on the main ideas. 
We discuss a method for removing this assumption in Appendix~\ref{appendix:removingHittingTime}.

Algorithm~\ref{algo:aomtrl} outlines the main steps of our approach. Given a set $\Gamma^\alpha$ of $\alpha$-distinguishing state-action pairs, in the clustering phase the learner employs a history-dependent policy specified by Algorithm~\ref{algo:ExploreID}, \texttt{ExploreID}, to collect at least $N$ samples for each state-action pair in $\Gamma^\alpha$, where $N$ will be determined later. Once all $(s,a)$ in $\Gamma^\alpha$ have been visited at least $N$ times, Algorithm~\ref{algo:IdentifyCluster}, \texttt{IdentifyCluster}, computes the empirical means of the transition function of these $(s,a)$ and then compares them with those in each cluster to determine which cluster contains the samples from the same task (or none do, in which case a new cluster is created). 
For the rest of the episode, the learner uses the UCBVI-CH algorithm~\citep{Azar2017} to learn the optimal policy.

The algorithms and results up to Theorem~\ref{theorem:mainRegret1} are presented for a general set $\Gamma^\alpha$. Since $\Gamma^\alpha$ is generally unknown, Corollary~\ref{corollary:mainRegret} shows the result for $\alpha = \lambda$ and $\Gamma^\alpha = \mathcal{S} \times \mathcal{A}$.

\begin{figure}[t]
  \centering
  \begin{minipage}{0.54\textwidth}
  \begin{algorithm2e}[H]
    \SetAlgoNoEnd
    \DontPrintSemicolon
    \KwIn{Number of models $M$, number of episodes $K$, MDPs parameters $\mathcal{S, A}, H, \tilde{D}, \lambda$, probability $p$, separation level $\alpha$ and an $\alpha$-distinguishing set $\Gamma^\alpha$.}
    Compute $p_1 = p/3, N = \frac{256}{\lambda^2}\max\{S, \ln(\frac{K|\Gamma^\alpha|}{p_1})\}, \delta = \alpha - \lambda/4, H_{0} = 12D|\Gamma^\alpha|N$\;
    Initialize $\mathcal{C} \leftarrow \emptyset$\;
    
    \For{$k = 1, \dots, K$}{
    Initialize $\mathcal{B}_k \leftarrow \emptyset$\;
    The environment chooses a task $m^k$\;
    Observe the initial state $s_1$\;
    \For{$h = 1, \dots, H_{0}$ }{  
        \SetKwFunction{ExploreID}{ExploreID}
        $a_h = $ \ExploreID{$s_h, \Gamma^\alpha$}\;
        Observe $s_{h+1}$ and $r_{h+1}$\;
        Add $(s_h, a_h, s_{h+1})$ to $\mathcal{B}_k$\;
    }
    \SetKwFunction{IdentifyCluster}{IdentifyCluster}
    $id \leftarrow$ \IdentifyCluster{$\mathcal{B}_k, \Gamma^\alpha, \mathcal{C}, \delta$}\;
    \lIf{$id \geq 1$}{
        $\mathcal{C}_{id}^{model} = \mathcal{C}_{id}^{model} \cup \mathcal{B}_k$
    }
    \uElse{
      $id \leftarrow |\mathcal{C}| + 1$\;
      $\mathcal{C}_{id}^{model} = \mathcal{B}_k$, $\mathcal{C}_{id}^{regret} = \emptyset$\;
      $\mathcal{C} \leftarrow \mathcal{C} \cup \mathcal{C}_{id}$\;
    }
    $\pi_k = \texttt{UCBVI-CH}(\mathcal{C}_{id}^{regret})$\;
    \For{$h = H_{0}+1, \dots, H$ }{
        $a_h = \pi_k(h, s_h)$\;
        Observe $s_{h+1}$ and $r_{h+1}$\;
        $\mathcal{C}_{id}^{regret} = \mathcal{C}_{id}^{regret} \cup (s_h, a_h, s_{h+1})$\;
    }
    }
    \caption{\protect Adversarial online multi-task RL}
    \label{algo:aomtrl}
    \end{algorithm2e}
  \end{minipage}
  \hfill
  \begin{minipage}{0.44\textwidth}
  \begin{algorithm2e}[H]
      \SetAlgoNoEnd
      \DontPrintSemicolon
      \KwIn{Episode $k$, state $s$, set $\Gamma^\alpha$ and number $N$}
      Set $\mathcal{G}(s) = \left\{ \begin{array}{l} a \in \mathcal{A} : \\ (s, a) \in \Gamma^\alpha, N_{\mathcal{B}_k}(s, a) < N \end{array} \right\}$\;
      \uIf{$\mathcal{G}(s) \neq \emptyset$}
      {
          \Return $\argmax_{a \in \mathcal{G}(s)}N_{\mathcal{B}_k}(s, a)$
      }
      \uElse{
          \Return $\argmax_{a \in \mathcal{A}} \sum_{s' = 1}^S \hat{P}^k(s' \mid s, a)\mathbb{I}\{\mathcal{G}(s') \neq \emptyset\}$
      }
      \caption{ExploreID}
      \label{algo:ExploreID}
      \end{algorithm2e}
     
      \vspace{5mm}
  
      \begin{algorithm2e}[H]
        \SetAlgoNoEnd
        \DontPrintSemicolon
        \KwIn{Episode $k$, set $\Gamma^\alpha$, clusters $\mathcal{C}$, and threshold $\delta$}
        \For{$c = 1, \dots, \norm{\mathcal{C}}$}
        {
            Initialize \texttt{id} $\leftarrow c$ \;
            \For{$(s, a) \in \Gamma$}
            {
                \uIf{$\norm{[\hat{P}_c - \hat{P}^k](s, a)} > \delta$}
                {
                    \texttt{id} $\leftarrow 0$\;
                    break;
                }
            }
            \uIf{\texttt{id} $ == c$}
            {
                \Return \texttt{id}; 
            }
        }
        \Return 0;
        \caption{Identify Cluster}
        \label{algo:IdentifyCluster}
        \end{algorithm2e}
  \vspace{19.9mm} 
  \end{minipage}
  \end{figure}

\subsection{The Exploration Algorithm}
\label{sec:theExplorationAlgorithm}
      Given a collection $\mathcal{B}$ of tuples $(s, a, s')$, the empirical transition functions estimated by $\mathcal{B}$ are
      \begin{align*}
          \hat{P}_{\mathcal{B}}(s' \mid s, a) =
          \begin{cases}
              \frac{N_{\mathcal{B}}(s, a, s')}{N_{\mathcal{B}}(s, a)} & \text{if }{N_{\mathcal{B}}(s,a) > 0} \\
              0 & \text{otherwise},
          \end{cases}
        \end{align*}
      
      \begin{align*}
          \text{where}
          \qquad \qquad
          N_{\mathcal{B}}(s, a, s') = \sum_{(x,y,z)\in \mathcal{B}}\mathbb{I}\{x = s, y = a, z = s'\}, \qquad
          N_{\mathcal{B}}(s, a) = \sum_{s' \in \mathcal{S}}N_{\mathcal{B}}(s, a, s') 
          \end{align*}
are the number of instances of $(s, a, s')$ and $(s, a)$ in $\mathcal{B}$, respectively.
      
      For each episode $k$, let $P^k$ denote the transition function of the task $m^k$ and $\mathcal{B}_k$ denote the collection of samples $(s_h, a_h, s_{h+1})$ collected during the learning phase.  The empirical means $\hat{P}^k$  estimated using samples in $\mathcal{B}_k$ are $\hat{P}^k = \hat{P}_{\mathcal{B}_k}$. The value of $N$ can be chosen so that for all $(s,a) \in \Gamma^\alpha$, with high probability $\hat{P}^k(s,a)$ is close to $P^k(s,a)$. Specifically, we find that if $N$ is large enough so that $\hat{P}^k(s,a)$ is within $\lambda/8$ in $\ell_1$ norm of the true function $P^k(s,a)$, then the right cluster can be identified in every episode.
       The exact value of $N$ is given in the following lemma.
      
      \begin{mylemma} Suppose the learner is given a constant $p_1 \in (0, 1)$ and a $\alpha$-distinguishing set $\Gamma^\alpha \subseteq \mathcal{S} \times \mathcal{A}$. If each state-action pair in $\Gamma^\alpha$ is visited at least $N = \frac{256}{\lambda^2}\max\{S, \ln(\frac{K|\Gamma^\alpha|}{p_1})\}$ times during the clustering phase of each episode $k=1, 2, \dots, K$, then with probability at least $1 - p_1$, the event
      \begin{align*}
      \mathcal{E}^{\Gamma^\alpha}_{k} = \left\{\forall (s, a) \in \Gamma^\alpha, \norm{P^k(s,a) - \hat{P}^k(s,a)} \leq \frac{\lambda}{8} \right\} \text{ holds for all } k \in [K].
      \end{align*}
      \label{lemma:goodeEstimatorId}
      \end{mylemma}
      The exploration in AOMultiRL is modelled as an instance of the active model estimation problem~\citep{Tarbouriech2020}.
      Given the current state $s$, if there exists an action $a$ such that $(s, a) \in \Gamma^\alpha$ and $(s,a)$ has not been visited at least $N$ times, this action will be chosen (with ties broken by selecting the most chosen action). Otherwise, the algorithm chooses an action that has the highest estimated probability of leading to an under-sampled state-action pair in $\Gamma^\alpha$. 
      The following lemma computes the number of steps $H_0$ in the clustering phase.
      
      \begin{restatable}{mylemma}{LemmaComputeH}
      Consider $p_1$ and $N$ defined in Lemma~\ref{lemma:goodeEstimatorId}. By setting 
      \begin{align*}
        H_0 = 12\tilde{D}|\Gamma^\alpha|N = \frac{3072\tilde{D}|\Gamma^\alpha|}{\lambda^2}\max\{S, \ln(\frac{K|\Gamma^\alpha|}{p_1})\},
      \end{align*} with probability at least $1 - p_1$, Algorithm~\ref{algo:ExploreID} visits each state-action pair in $\Gamma^\alpha$ at least $N$ times during the clustering phase in each of the $K$ episodes.    
      \label{lemma:computeH0}
      \end{restatable}
      
      \subsection{The Clustering Algorithm}
      \label{sec:clusteringAlgorithm}
      
      Denote by $\mathcal{C}$ the set of clusters, $C = |\mathcal{C}|$ the number of clusters and $\mathcal{C}_i$ the $i\nth$ cluster. Each $\mathcal{C}_i$ is a collection of two multisets $\mathcal{C}_i^{model}, \mathcal{C}_i^{regret} \subset \mathcal{S \times A \times S}$ which contain the $(s, a, s')$ samples collected during the clustering and learning phases, respectively. Formally, up to episode $k$ we have
      \begin{align*}
              \mathcal{C}_i^{model} &= \cup_{k' = 1}^{k-1}\{(s^{k'}_h, a^{k'}_h, s^{k'}_{h+1}):h \leq H_0, \mathrm{id}^{k'}=i\}, \\
              \mathcal{C}_i^{regret} &= \cup_{k' = 1}^{k-1}\{(s^{k'}_h, a^{k'}_h, s^{k'}_{h+1}):h > H_0, \mathrm{id}^{k'}=i\},
      \end{align*} where $s^k_h$ and $a^k_h$ are the state and action at time step $h$ of episode $k$, respectively and $id^{k'}$ is the cluster index returned by Algorithm~\ref{algo:IdentifyCluster} in episode $k'$.

      Let $\hat{P}_i = \hat{P}_{\mathcal{C}_i^{model}}$ denote the empirical means estimated using samples in $\mathcal{C}_i^{model}$. For each episode $k$, from Lemma~\ref{lemma:computeH0} with high probability after the first $H_0$ steps each state-action pair $(s,a) \in \Gamma^\alpha$ has been visited at least $N$ times. 
Algorithm~\ref{algo:IdentifyCluster} determines the right cluster for a task by computing the $\ell_1$ distance between $\hat{P}^k$ and the empirical transition function $\hat{P}_i$ for each cluster $i = 1, 2, \dots, C$. 
If there exists an $(s,a) \in \Gamma^\alpha$ such that the distance is larger than a certain threshold $\delta$, i.e., $\norm{[\hat{P}_i - \hat{P}^k](s, a)} > \delta$, then the algorithm concludes that the task belongs to another cluster. Otherwise, the task is considered to belong to cluster $i$. We set $\delta = \alpha - \lambda/4$. The following lemma shows that with this choice of $\delta$, the right cluster is identified by Algorithm~\ref{algo:IdentifyCluster} in all episodes.

\begin{mylemma} Consider a $\lambda$-separable set of MDPs $\mathcal{M}$ and an $\alpha$-distinguishing set $\Gamma^\alpha$ where $\alpha \geq \lambda/2$. If the events $\mathcal{E}^{\Gamma^\alpha}_k$ defined in Lemma~\ref{lemma:goodeEstimatorId} hold for all $k \in [K]$, then with the distance threshold $\delta = \alpha - \lambda/4$ Algorithm~\ref{algo:IdentifyCluster} always produces a correct output in each episode: the trajectories of the same model in two different episodes are clustered together and no two trajectories of two different models are in the same cluster.
\label{lemma:noIncorrectClusters}
\end{mylemma}

Once the clustering phase finishes, the learner enters the learning phase and uses the UCBVI-CH algorithm~\citep{Azar2017} to learn the optimal policy for this phase. In principle, almost all standard single-task RL algorithms with a near-optimal regret guarantee can be used for this phase. We chose UCBVI-CH to simplify the analysis and make the exposition clear.

To simulate the standard single-task episodic learning setting, the learner only uses the samples in $\mathcal{C}_i^{regret}$ for regret minimization.  
The impact of combining samples in two phases for regret minimization is addressed in Appendix~\ref{sec:externalsamples}. 
Theorem~\ref{theorem:mainRegret1} states a regret bound for Algorithm~\ref{algo:aomtrl}.

\begin{restatable}{theorem}{TheoremMainRegret}
For any failure probability $p \in (0, 1)$, with probability at least $1-p$ the regret of Algorithm~\ref{algo:aomtrl} is bounded as
\[
    \begin{split}
\text{Regret}(K) \leq 2KH_0 &+ 67H_1^{3/2}L\sqrt{MSAK} + 15MS^2AH_1^2L^2,
    \end{split}
\]
where $H_0 = 12\tilde{D}|\Gamma^\alpha|N$, $N = \frac{256}{\lambda^2}\max\{S, \ln(\frac{3K|\Gamma^\alpha|}{p})\}$, $H_1 = H - H_0$, and $L = \ln(15SAKHM/p)$.
\label{theorem:mainRegret1}
\end{restatable}

For $K > MS^3AH$, the first two terms are the most significant. The $2KH_0$ term accounts for the clustering phase and the fact that the exploration policy might lead the learner to an undesirable state after $H_0$ steps. The $\tilde{O}(\sqrt{K})$ term comes from the fact that the learning phase is equivalent to episodic single-task learning with horizon $H_1$.
When $H \gg H_0$, the sub-linear bound on the learning phase is a major improvement compared to the $O(K\sqrt{HSA})$ bound of the strategy that learns each episode individually. 

By setting $\Gamma^\alpha = \mathcal{S} \times \mathcal{A}$ and $\alpha = \lambda$, we obtain

\begin{mycorollary}
  For any failure probability $p \in (0, 1)$, with probability at least $1-p$, by setting $\Gamma^\alpha = \mathcal{S} \times \mathcal{A}$ with $\alpha = \lambda$, the regret of Algorithm~\ref{algo:aomtrl} is
\begin{align}
\text{Regret}(K) \leq O\left(\frac{K\tilde{D}SA}{\lambda^2}\ln\left(\frac{KSA}{p}\right) + H^{3/2}L\sqrt{MSAK}\right).
\end{align}
where $L = \ln(15SAKH_1M/p)$.
\label{corollary:mainRegret}
\end{mycorollary}

\textbf{Time Complexity} The clustering algorithm runs once in each episode, which leads to time complexity of $O(MSA + H)$. When $H \gg H_0$, the overall time complexity is dominated by the learning phase, which is $O(HSA)$ for UCBVI-CH.

\begin{remark}
  Instead of clustering, a different paradigm involves actively merging samples from different MDPs to learn a model that is an averaged estimate of the MDPs in $\mathcal{M}$. The best regret guarantee in this paradigm, to the best of our knowledge, is $\tilde{O}(S^{1/3}A^{1/3}B^{1/3}H^{5/3}K^{2/3})$, where $B$ is a variation budget, achieved by RestartQ-UCB~\citep[Theorem 3]{Mao2021b}. In our setting, if the adversary frequently alternates between tasks then $B = \Omega(KH\lambda)$ and therefore this bound becomes $\tilde{O}(\lambda^{1/3}S^{1/3}A^{1/3}H^{2}K)$, which is larger than the trivial bound $KH$ and worse than the bound in Corollary~\ref{corollary:mainRegret}. If the adversary selects tasks so that $B$ is small i.e. $B = o(K)$ then the bound offered by RestartQ-UCB is better since it is sub-linear in $K$. Note that this does not contradict the lower bound result in Section~\ref{sec:lowerbounds}, since the lower bound is constructed with an adversary that selects tasks uniformly at random, and hence $B$ is linear in $K$.
\end{remark}

\subsection{Learning a distinguishing set when $M$ is small}
\label{sec:learnGamma}

As pointed out by~\citet{BrunskillAndLi2013}, for all $\alpha > 0$, the size of the smallest $\alpha$-distinguishing set of $\mathcal{M}$ is at most $M \choose 2$. If $M^2 \ll SA$ and such a set is known to the learner, then the clustering phase only need collect samples from this set instead of the full $\mathcal{S}\times\mathcal{A}$ set of state-action pairs. However, in general this set is not known. We show that if the adversary is weaker so that all models are guaranteed to appear at least once early on, the learner will be able to discover a $\frac{\lambda}{2}$-distinguishing set $\hat{\Gamma}$ of size at most ${M \choose 2}$. Specifically, we employ the following assumption:

\begin{assumption}
There exists an unknown constant $K_1 \geq M$ satisfying $ K_1SA < K$ such that after at most $K_1$ episodes, each model in $\mathcal{M}$ has been given to the learner at least once.
\label{assumption:K1}
\end{assumption}

In order to discover $\hat{\Gamma}$, the learner uses Algorithm~\ref{algo:unknownGamma}, which consists of two stages:

\begin{itemize}[leftmargin=6mm]
    \item Stage 1: the learner starts by running Algorithm~\ref{algo:aomtrl} with the $\lambda$-distinguishing set candidate $\mathcal{S \times A}$ until the number of clusters is $M$. With high probability, each cluster corresponds to a model. 
    At the end of stage 1, the learner uses the empirical estimates in all clusters $\hat{P}_i$ for $i \in [M]$ to construct a $\lambda/2$-distinguishing set $\hat{\Gamma}$ for $\mathcal{M}$. 
    \item Stage 2: the learner runs Algorithm~\ref{algo:aomtrl} with the distinguishing set $\hat{\Gamma}$ as an input.
\end{itemize}

\textbf{Extracting $\lambda/2$-distinguishing pairs:} After $K_1$ episodes, with high probability there are $M$ clusters corresponding to $M$ models. 
For two clusters $i$ and $j$, the set $\hat{\Gamma}_{i,j}$ contains the first state-action pair $(s, a)$ that satisfies $\norm{\hat{P}_i(s,a) - \hat{P}_j(s,a)} > 3\lambda/4$. With high probability, every $(s, a) \in \Gamma_{i,j}$ satisfies this condition, hence $\hat{\Gamma}_{i,j} \neq \emptyset$.

Let $i^\star \in [M]$ denote the index of the MDP model corresponding to cluster $i$.
For all $(s, a) \in \hat{\Gamma}_{i,j}$, by the triangle inequality, we have
\begin{align*}
\norm{P_{i^\star} - P_{j^\star}} \geq \norm{\hat{P}_i - \hat{P}_j} - \norm{\hat{P}_i - P_{i^\star} + P_{j^\star} - \hat{P}_j}            
            > 3\lambda/4 - (\lambda/8 + \lambda/8)
            = \lambda/2, 
\end{align*}
where $(s,a)$ is omitted for brevity.
It follows that the set $\hat{\Gamma} = \cup_{i,j}\hat{\Gamma}_{i,j}$ is $\lambda/2$-distinguishing and $|\hat{\Gamma}| \leq {M \choose 2}$. Although $\lambda/2$ is smaller than the $\lambda$-separation level of $\Gamma$, it is sufficient for the conditions in Lemma~\ref{lemma:noIncorrectClusters} to hold. Thus, with high probability the clustering algorithm in stage 2 works correctly. The next theorem shows the regret guarantee of Algorithm~\ref{algo:unknownGamma}.

\begin{restatable}{theorem}{TheoremRegretAlgorithmNoGamma}
  Under Assumption~\ref{assumption:K1}, With probability at least $1 - p$, the regret of Algorithm~\ref{algo:unknownGamma} is
  \begin{align*}
    \textstyle
    \text{Regret}(K) = O\left(\frac{K\tilde{D}M^2}{\lambda^2}\ln{\frac{KM^2}{p}} + H^{3/2}L\sqrt{MKSA}\right),
  \end{align*}
  where $H_{0, M} = \frac{3072\tilde{D}M^2}{\lambda^2}\max\{S, \ln(\frac{3KM^2}{p})\}$ and $L = \ln(15SAKH_1M/p)$.
  \label{theorem:regretUnknownGamma}
  \end{restatable}

  Compared to Corollary~\ref{corollary:mainRegret}, Theorem~\ref{theorem:regretUnknownGamma} improves the clustering phase's dependency from $SA$ to $M^2$. This implies that if the number of models is small and all models appear relatively early, we can discover a $\lambda/2$-distinguishing set quickly without increasing the order of the total regret bound.

\begin{algorithm2e}[tbp]
    \SetAlgoNoEnd
    \KwIn{Number of models $M$, number of episodes $K$, MDPs parameters $\mathcal{S, A}, H, \tilde{D}, \lambda$, probability $p$}

Stage 1: Run Algorithm~\ref{algo:aomtrl} with the distinguishing set $\Gamma^\alpha = \mathcal{S \times A}$ and $\alpha = \lambda$ until the number of clusters is $M$\;
\For{$i, j \in [M] \times [M], i \neq j$}
{
    $\hat{\Gamma}_{i,j} = \emptyset$\;
    \For{$(s,a) \in \mathcal{S \times A}$}
    {
        \uIf{$\norm{\hat{P}_i(s,a) - \hat{P}_j(s,a)} > 3\lambda/4$}
        {
            $\hat{\Gamma}_{i,j} = \hat{\Gamma}_{i,j} \cup (s,a)$\;
            break
        }
    }
}
$\hat{\Gamma} = \cup_{i,j}\hat{\Gamma}_{i,j}$\;
Stage 2: Run Algorithm~\ref{algo:aomtrl} with distinguishing set $\hat{\Gamma}$ and $\alpha = \lambda / 2$ for $K_2 = K - K_1$ episodes.

\caption{AOMultiRL with all models being given at least once}
\label{algo:unknownGamma}
\end{algorithm2e}

\section{Experiments}
\label{sec:experiments}
\begin{figure}
    \centering
    \includegraphics[width=0.8\linewidth]{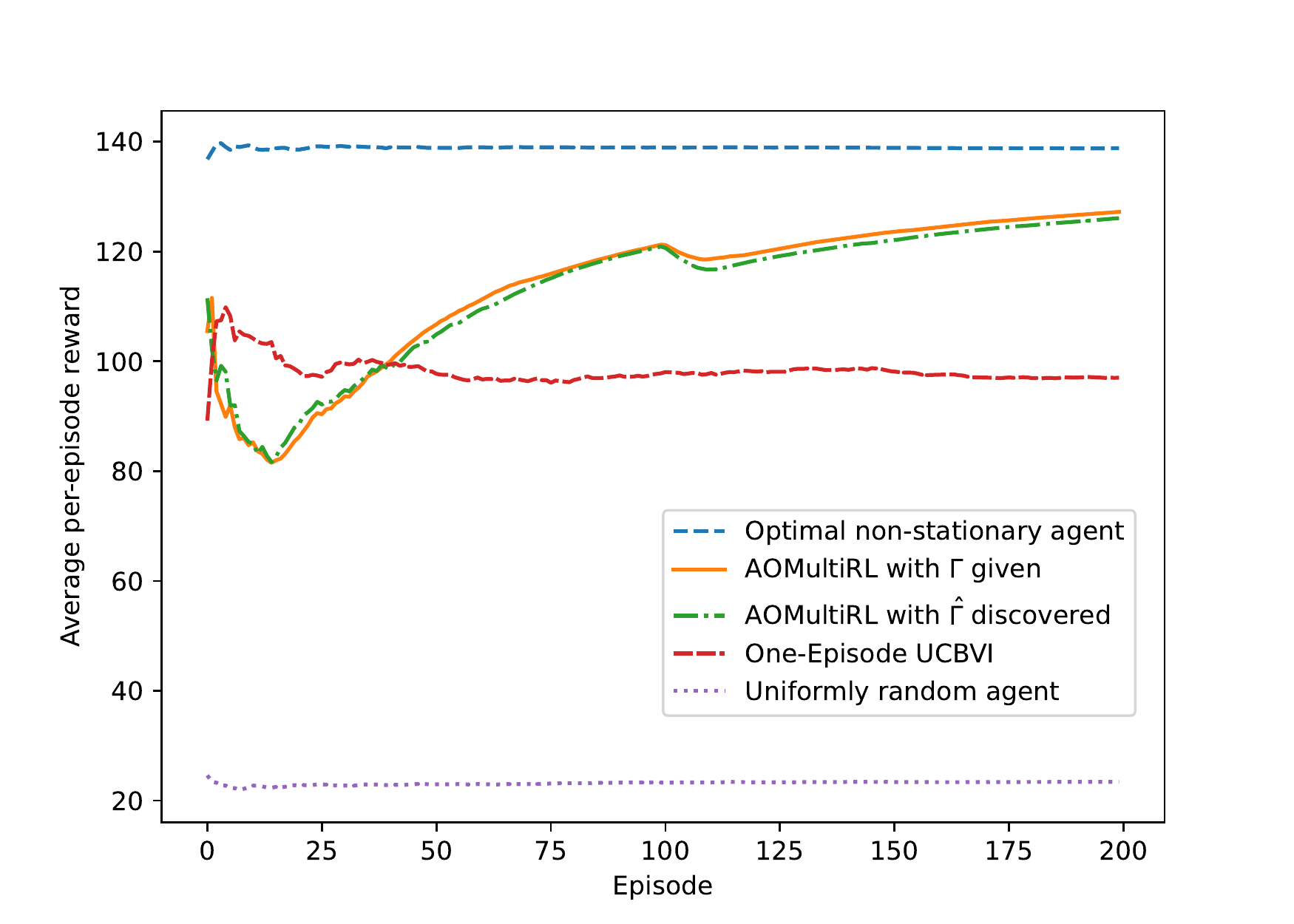}
    \caption{Average per-episode reward.}\label{fig:result}
  \end{figure}
We evaluate AOMultiRL on a sequence of $K=200$ episodes, where the task in each episode is taken from a set of $M = 4$ MDPs.
Each MDP in $\mathcal{M}$ is a $4 \times 4$ grid of $S = 16$ cells with $A = 4$ valid actions: \texttt{up, down, left, right}. 
The state for row $r$ and column $c$ (0-indexed) is represented by the tuple $(r,c)$. The reward is 0 in every state, except for the four corners $(0,0), (0,3), (3,0)$, and $(3,3)$, where the reward is 1. We fix the initial state at $(1,1)$. 

To simulate an adversarial sequence of tasks, episodes 100 to 150 and episodes 180 to 200 contain only the MDP $m_4$. Other episodes chooses $m_1, m_2$ and $m_3$ uniformly at random. The hitting time is $D = 7$ and the failure probability is $p = 0.03$. We use the \texttt{rlberry} framework~\citep{rlberry} for our implementation\footnote{The code is available at \url{https://github.com/ngmq/adversarial-online-multi-task-reinforcement-learning}}.

We construct the transition functions so that each MDP has only one easy-to-reach corner, which corresponds to a unique optimal policy. The separation level $\lambda$ is $1.2999$. Furthermore, there exists state-action pairs that are $\lambda/2$-distinguishing but not $\lambda$-distinguishing. More details can be found in Appendix~\ref{appendix:expsetup}.

The baseline algorithms include a random agent that chooses actions uniformly at random, a one-episode UCBVI agent which does not group episodes and learns using only the samples of each episode, and the optimal non-stationary agent that acts optimally in every episode. The first and the last baselines serve as the lower bound and upper bound performance for AOMultiRL, while the second baseline helps illustrate the effectiveness of clustering episodes correctly. We evaluate two instances of AOMultiRL: AOMultiRL1 with a set $\Gamma$ of $|\Gamma| = 3$ given and AOMultiRL2 without any distinguishing set given. 
We follow the approach of~\cite{LatentMDPs2021} and evaluate all five algorithms based on their expected cumulative reward when starting at state $(1,1)$ and following their learned policy for $H_1 = 200$ steps (averaged over 10 runs). While the horizon for the learning phase is much smaller than the horizon for clustering phase of $H_0 \approx 80000$, we ensure the fairness of the comparisons by not using the samples collected in the clustering phase in the learning phase, thus simulating the setting where $H_0 \ll H_1$ without the need to use significantly larger MDPs. We use the average per-episode reward as the performance metric. Figure~\ref{fig:result} shows the results.

\textbf{The effectiveness of the clustering on the learning phase.} To measure the effectiveness of aggregating samples from episodes of the same task for the learning phase, we compare AOMultiRL1 and the one-episode UCBVI agent. Since for every pair of MDP models, the transition functions are distinct for state-action pairs adjacent to two of the corners, AOMultiRL1 can only learn the estimated model accurately for each MDP model if the clustering phase produces correct clusters in most of the episodes. 
We can observe in Figure~\ref{fig:result} that after about thirty episodes, AOMultiRL1 starts outperforming the one-episode UCBVI agent and approaching the performance of the optimal non-stationary agent. 
The model $m_4$ appears for the first time in episode 100, which accounts for the sudden drop in performance in that episode. 
Afterwards, the performance of AOMultiRL1 steadily increases again. This demonstrates that the AOMultiRL1 is able to identify the correct cluster in most of the episodes, which enables the multi-episode UCBVI algorithm in AOMultiRL1 to estimate the MDP models much more accurately than the non-transfer one-episode UCBVI agent. This suggests that for larger MDPs where $H_1 \gg H_0$, spending a number of initial steps on finding the episodes of the same task would yield higher long-term rewards.

\textbf{Performance of AOMultiRL with the discovered $\bm{\hat{\Gamma}}$}. Next, we examine the performance of AOMultiRL2 when no distinguishing set is given. We run AOMultiRL2 for 204 episodes, in which stage 1 consists of the first four episodes, each containing one of the four MDP models in $\mathcal{M}$.
As the identities of the models are not given, the algorithm has to correctly construct four clusters and then compute a $\lambda/2$-distinguishing set after the $4\nth$ episode even though each model is seen just once. As mentioned above, the MDPs are set up so that if the AOMultiRL2 correctly identifies four clusters, then the discovered $\hat{\Gamma}$ will contain at least one state-action pair that is $\lambda/2$-distinguishing but not $\lambda$-distinguishing. 
In stage 2, the horizon of the learning phase is set to the same $H_1 = 200$ used for AOMultiRL1.
The performance in stage 2 of AOMultiRL2 approaches that of AOMultiRL1, indicating that the discovered $\hat{\Gamma}$ is as effective as the set $\Gamma$ given to AOMultiRL1.

\section{Conclusion}
In this paper, we studied the adversarial online multi-task RL setting with the tasks belonging to a finite set of well-separated models. We used a general notion of task-separability, which we call $\lambda$-separability. Under this notion, we proved a minimax regret lower bound that applies to all algorithms and an instance-specific regret lower bound that applies to a class of uniformly good cluster-then-learn algorithms. We further proposed AOMultiRL, a polynomial time cluster-then-learn algorithm that obtains a nearly-optimal instance-specific regret upper bound. These results addressed two fundamental aspects of online multi-task RL, namely learning an adversarial task sequence and learning under a general task-separability notion. Adversarial online multi-task learning remains challenging when the diameter and the number of models are unknown; this is left for future work.

\section*{Acknowledgement}
This work was supported by the NSERC Discovery Grant RGPIN-2018-03942.

\bibliographystyle{plainnat}
\bibliography{refs.bib}

\begin{thebibliography}{29}
\providecommand{\natexlab}[1]{#1}
\providecommand{\url}[1]{\texttt{#1}}
\expandafter\ifx\csname urlstyle\endcsname\relax
  \providecommand{\doi}[1]{doi: #1}\else
  \providecommand{\doi}{doi: \begingroup \urlstyle{rm}\Url}\fi

\bibitem[Abel et~al.(2018)Abel, Jinnai, Guo, Konidaris, and
  Littman]{AbelLLRL2018}
David Abel, Yuu Jinnai, Sophie~Yue Guo, George Konidaris, and Michael Littman.
\newblock Policy and value transfer in lifelong reinforcement learning.
\newblock In Jennifer Dy and Andreas Krause, editors, \emph{Proceedings of the
  35th International Conference on Machine Learning}, volume~80 of
  \emph{Proceedings of Machine Learning Research}, pages 20--29. PMLR, 10--15
  Jul 2018.

\bibitem[Auer and Ortner(2007)]{UCRL2006}
Peter Auer and Ronald Ortner.
\newblock Logarithmic online regret bounds for undiscounted reinforcement
  learning.
\newblock In B.~Sch\"{o}lkopf, J.~Platt, and T.~Hoffman, editors,
  \emph{Advances in Neural Information Processing Systems}, volume~19. MIT
  Press, 2007.

\bibitem[Auer et~al.(2002)Auer, Cesa-Bianchi, Freund, and
  Schapire]{EXP3Auer2002b}
Peter Auer, Nicol\`{o} Cesa-Bianchi, Yoav Freund, and Robert~E. Schapire.
\newblock The nonstochastic multiarmed bandit problem.
\newblock \emph{SIAM Journal on Computing}, 32\penalty0 (1):\penalty0 48--77,
  2002.
\newblock \doi{10.1137/S0097539701398375}.
\newblock URL \url{https://doi.org/10.1137/S0097539701398375}.

\bibitem[Azar et~al.(2013)Azar, Lazaric, and Brunskill]{Azar2013}
Mohammad~Gheshlaghi Azar, Alessandro Lazaric, and Emma Brunskill.
\newblock Sequential transfer in multi-armed bandit with finite set of models.
\newblock In C.~J.~C. Burges, L.~Bottou, M.~Welling, Z.~Ghahramani, and K.~Q.
  Weinberger, editors, \emph{Advances in Neural Information Processing
  Systems}, volume~26. Curran Associates, Inc., 2013.

\bibitem[Azar et~al.(2017)Azar, Osband, and Munos]{Azar2017}
Mohammad~Gheshlaghi Azar, Ian Osband, and R{\'e}mi Munos.
\newblock Minimax regret bounds for reinforcement learning.
\newblock In Doina Precup and Yee~Whye Teh, editors, \emph{Proceedings of the
  34th International Conference on Machine Learning}, volume~70 of
  \emph{Proceedings of Machine Learning Research}, pages 263--272. PMLR, 06--11
  Aug 2017.

\bibitem[Brunskill and Li(2013)]{BrunskillAndLi2013}
Emma Brunskill and Lihong Li.
\newblock Sample complexity of multi-task reinforcement learning.
\newblock In \emph{Proceedings of the Twenty-Ninth Conference on Uncertainty in
  Artificial Intelligence}, UAI'13, page 122–131, Arlington, Virginia, USA,
  2013. AUAI Press.

\bibitem[Brunskill and Li(2015)]{BrunskillAndLi2015}
Emma Brunskill and Lihong Li.
\newblock The online discovery problem and its application to lifelong
  reinforcement learning.
\newblock \emph{CoRR}, abs/1506.03379, 2015.

\bibitem[Cover and Thomas(2006)]{CoverAndThomas2006}
Thomas~M. Cover and Joy~A. Thomas.
\newblock \emph{Elements of Information Theory (Wiley Series in
  Telecommunications and Signal Processing)}.
\newblock Wiley-Interscience, USA, 2006.
\newblock ISBN 0471241954.

\bibitem[Dann and Brunskill(2015)]{Dann2015}
Christoph Dann and Emma Brunskill.
\newblock Sample complexity of episodic fixed-horizon reinforcement learning.
\newblock In C.~Cortes, N.~Lawrence, D.~Lee, M.~Sugiyama, and R.~Garnett,
  editors, \emph{Advances in Neural Information Processing Systems}, volume~28.
  Curran Associates, Inc., 2015.

\bibitem[Dann et~al.(2017)Dann, Lattimore, and
  Brunskill]{DannUnifyingPACandRegret2017}
Christoph Dann, Tor Lattimore, and Emma Brunskill.
\newblock Unifying {PAC} and regret: Uniform {PAC} bounds for episodic
  reinforcement learning.
\newblock In I.~Guyon, U.~Von Luxburg, S.~Bengio, H.~Wallach, R.~Fergus,
  S.~Vishwanathan, and R.~Garnett, editors, \emph{Advances in Neural
  Information Processing Systems}, volume~30. Curran Associates, Inc., 2017.

\bibitem[Domingues et~al.(2021{\natexlab{a}})Domingues, Flet-Berliac, Leurent,
  M{\'e}nard, Shang, and Valko]{rlberry}
Omar~Darwiche Domingues, Yannis Flet-Berliac, Edouard Leurent, Pierre
  M{\'e}nard, Xuedong Shang, and Michal Valko.
\newblock {rlberry - A Reinforcement Learning Library for Research and
  Education}, 10 2021{\natexlab{a}}.
\newblock URL \url{https://github.com/rlberry-py/rlberry}.

\bibitem[Domingues et~al.(2021{\natexlab{b}})Domingues, M{\'e}nard, Kaufmann,
  and Valko]{Domingues2021}
Omar~Darwiche Domingues, Pierre M{\'e}nard, Emilie Kaufmann, and Michal Valko.
\newblock Episodic reinforcement learning in finite {MDP}s: Minimax lower
  bounds revisited.
\newblock In Vitaly Feldman, Katrina Ligett, and Sivan Sabato, editors,
  \emph{Proceedings of the 32nd International Conference on Algorithmic
  Learning Theory}, volume 132 of \emph{Proceedings of Machine Learning
  Research}, pages 578--598. PMLR, 16--19 Mar 2021{\natexlab{b}}.
\newblock URL \url{https://proceedings.mlr.press/v132/domingues21a.html}.

\bibitem[Guo and Brunskill(2015)]{GuoAndBrunskill2015}
Zhaohan Guo and Emma Brunskill.
\newblock Concurrent {PAC} {RL}.
\newblock In \emph{Proceedings of the Twenty-Ninth AAAI Conference on
  Artificial Intelligence}, AAAI'15, page 2624–2630. AAAI Press, 2015.
\newblock ISBN 0262511290.

\bibitem[Hallak et~al.(2015)Hallak, Castro, and Mannor]{Hallak2015contextual}
Assaf Hallak, Dotan~Di Castro, and Shie Mannor.
\newblock Contextual {M}arkov {D}ecision {P}rocesses, 2015.

\bibitem[Jaksch et~al.(2010)Jaksch, Ortner, and Auer]{UCRL2}
Thomas Jaksch, Ronald Ortner, and Peter Auer.
\newblock Near-optimal regret bounds for reinforcement learning.
\newblock \emph{Journal of Machine Learning Research}, 11\penalty0
  (51):\penalty0 1563--1600, 2010.

\bibitem[Jin et~al.(2018)Jin, Allen-Zhu, Bubeck, and Jordan]{ChiJin2018}
Chi Jin, Zeyuan Allen-Zhu, Sebastien Bubeck, and Michael~I Jordan.
\newblock Is {Q}-learning provably efficient?
\newblock In S.~Bengio, H.~Wallach, H.~Larochelle, K.~Grauman, N.~Cesa-Bianchi,
  and R.~Garnett, editors, \emph{Advances in Neural Information Processing
  Systems}, volume~31. Curran Associates, Inc., 2018.

\bibitem[Kwon et~al.(2021)Kwon, Efroni, Caramanis, and Mannor]{LatentMDPs2021}
Jeongyeol Kwon, Yonathan Efroni, Constantine Caramanis, and Shie Mannor.
\newblock {RL} for latent {MDP}s: Regret guarantees and a lower bound.
\newblock In \emph{Thirty-Fifth Conference on Neural Information Processing
  Systems}, 2021.

\bibitem[Lai and Robbins(1985)]{LaiAndRobbins1985}
Tze~Leung Lai and Herbert Robbins.
\newblock Asymptotically efficient adaptive allocation rules.
\newblock \emph{Advances in Applied Mathematics}, 6:\penalty0 4--22, 1985.

\bibitem[Lecarpentier et~al.(2021)Lecarpentier, Abel, Asadi, Jinnai, Rachelson,
  and Littman]{Lecarpentier2021LipschitzLR}
Erwan Lecarpentier, David Abel, Kavosh Asadi, Yuu Jinnai, Emmanuel Rachelson,
  and Michael~L. Littman.
\newblock Lipschitz lifelong reinforcement learning.
\newblock In \emph{AAAI}, 2021.

\bibitem[Levin et~al.(2008)Levin, Peres, and
  Wilmer]{MarkovChainAndMixingTime2008}
David~Asher Levin, Yuval Peres, and Elizabeth~Lee Wilmer.
\newblock \emph{{M}arkov Chains and Mixing Times}.
\newblock American Mathematical Soc., 2008.
\newblock ISBN 9780821886274.
\newblock URL \url{http://pages.uoregon.edu/dlevin/MARKOV/}.

\bibitem[Mao et~al.(2021)Mao, Zhang, Zhu, Simchi-Levi, and Basar]{Mao2021b}
Weichao Mao, Kaiqing Zhang, Ruihao Zhu, David Simchi-Levi, and Tamer Basar.
\newblock Near-optimal model-free reinforcement learning in non-stationary
  episodic {MDP}s.
\newblock In Marina Meila and Tong Zhang, editors, \emph{Proceedings of the
  38th International Conference on Machine Learning}, volume 139 of
  \emph{Proceedings of Machine Learning Research}, pages 7447--7458. PMLR,
  18--24 Jul 2021.
\newblock URL \url{http://proceedings.mlr.press/v139/mao21b.html}.

\bibitem[Sason(2015)]{sason2015reverse}
Igal Sason.
\newblock On reverse {P}insker inequalities.
\newblock \emph{arXiv preprint arXiv:1503.07118}, 2015.

\bibitem[Simchowitz and Jamieson(2019)]{Simchowitz2019}
Max Simchowitz and Kevin~G Jamieson.
\newblock Non-asymptotic gap-dependent regret bounds for tabular {MDP}s.
\newblock In H.~Wallach, H.~Larochelle, A.~Beygelzimer, F.~d\textquotesingle
  Alch\'{e}-Buc, E.~Fox, and R.~Garnett, editors, \emph{Advances in Neural
  Information Processing Systems}, volume~32. Curran Associates, Inc., 2019.

\bibitem[Steimle et~al.(2021)Steimle, Kaufman, and Denton]{Steimle2021}
Lauren~N. Steimle, David~L. Kaufman, and Brian~T. Denton.
\newblock Multi-model {M}arkov decision processes.
\newblock \emph{IISE Transactions}, 53\penalty0 (10):\penalty0 1124--1139,
  2021.
\newblock \doi{10.1080/24725854.2021.1895454}.

\bibitem[Sun and Huang(2020)]{Sun2020}
Yanchao Sun and Furong Huang.
\newblock Can agents learn by analogy? {A}n inferable model for {PAC}
  reinforcement learning.
\newblock In \emph{Proceedings of the 19th International Conference on
  Autonomous Agents and MultiAgent Systems}, AAMAS '20, page 1332–1340,
  Richland, SC, 2020. International Foundation for Autonomous Agents and
  Multiagent Systems.
\newblock ISBN 9781450375184.

\bibitem[Tarbouriech et~al.(2020)Tarbouriech, Shekhar, Pirotta, Ghavamzadeh,
  and Lazaric]{Tarbouriech2020}
Jean Tarbouriech, Shubhanshu Shekhar, Matteo Pirotta, Mohammad Ghavamzadeh, and
  Alessandro Lazaric.
\newblock Active model estimation in {M}arkov {D}ecision {P}rocesses.
\newblock In \emph{Uncertainty in Artificial Intelligence}, 2020.
\newblock URL
  \url{http://proceedings.mlr.press/v124/tarbouriech20a/tarbouriech20a-supp.pdf}.

\bibitem[Tarbouriech et~al.(2021)Tarbouriech, Pirotta, Valko, and
  Lazaric]{Tarbouriech2021a}
Jean Tarbouriech, Matteo Pirotta, Michal Valko, and Alessandro Lazaric.
\newblock A provably efficient sample collection strategy for reinforcement
  learning.
\newblock In A.~Beygelzimer, Y.~Dauphin, P.~Liang, and J.~Wortman Vaughan,
  editors, \emph{Advances in Neural Information Processing Systems}, 2021.
\newblock URL \url{https://openreview.net/forum?id=AvVDR8R-kQX}.

\bibitem[Tulsiani(2014)]{Tulsiani2014L6}
Madhur Tulsiani.
\newblock Lecture 6, lecture notes in information and coding theory, October
  2014.
\newblock URL
  \url{https://home.ttic.edu/~madhurt/courses/infotheory2014/l6.pdf}.

\bibitem[Weissman et~al.(2003)Weissman, Ordentlich, Seroussi, Verdu, and
  Weinberger]{weissman2003inequalities}
Tsachy Weissman, Erik Ordentlich, Gadiel Seroussi, Sergio Verdu, and Marcelo~J
  Weinberger.
\newblock Inequalities for the l1 deviation of the empirical distribution.
\newblock \emph{Hewlett-Packard Labs, Tech. Rep}, 2003.

\end{thebibliography}

\appendix

\section{Related Work}
\label{appendix:RelatedWorks}

\textbf{Stochastic Online Multi-task RL.} The Finite-Model-RL algorithm~\citep{BrunskillAndLi2013} considered the stochastic setting with infinite-horizon MDPs and focused on deriving a sample complexity of exploration in a $(\epsilon,\delta)$-PAC setting. As shown by~\citet{DannUnifyingPACandRegret2017}, even an optimal $(\epsilon, \delta)$-PAC bound can only guarantee a necessarily sub-optimal $O(K_m^{2/3})$ regret bound for each task $m \in [M]$ that appears in $K_m$ episodes, leading to an overall $O(M^{1/3}K^{2/3})$ regret bound for the learning phase in the multi-task setting.

The Contextual MDPs algorithm by~\citet{Hallak2015contextual} is capable of obtaining a $O(\sqrt{K})$ regret bound in the learning phase after the right cluster has been identified; however their clustering phase has exponential time complexity in $K$.
The recent L-UCRL algorithm~\citep{LatentMDPs2021} considered the stochastic finite-horizon setting and reduced the problem to learning the optimal policy of a POMDP. Under a set of assumptions that allow the clusters to be discovered in $O(\polylog(MSA))$, L-UCRL is able to obtain an overall $O(\sqrt{MK})$ regret with respect to a POMDP planning oracle which aims to learn a policy that maximizes the expected single-task return when a task is randomly drawn from a known distribution of tasks. In contrast, our work adopts a stronger notion of regret that encourages the learner to maximize its expected return for a sequence of tasks chosen by an adversary. When the models are bandits instead of MDPs,~\citet{Azar2013} use spectral learning to estimate the mean reward of the arms in all models and obtains an upper bound linear in $K$.

\vspace{\baselineskip}
\noindent\textbf{Lifelong RL.} Learning a sequence of related tasks is more well-studied in the lifelong learning literature. Recent works in lifelong RL~\citep{AbelLLRL2018,Lecarpentier2021LipschitzLR} often focus on the setting where tasks are drawn from an unknown distribution of MDPs and there exists some similarity measure between MDPs that support transfer learning. Our work instead focuses on learning the dissimilarity between tasks for the clustering phase and avoiding negative transfer.

\vspace{\baselineskip}
\noindent\textbf{Active model estimation} The exploration in AOMultiRL is modelled after the active model estimation problem~\citep{Tarbouriech2020}, which is often presented in PAC-RL setting. Several recent works on active model estimation are~PAC-Explore~\citep{GuoAndBrunskill2015}, FW-MODEST~\citep{Tarbouriech2020}, $\beta-$curious walking~\citep{Sun2020}, and GOSPRL~\citep{Tarbouriech2021a}. 

The $\Theta(\tilde{D}|\Gamma^\alpha|N)$ bound on the horizon of clustering in Lemma~\ref{lemma:computeH0} has the same $O(S^2A)$ dependency on the number of states and actions as the state-of-the-art bound by GOSPRL~\citep{Tarbouriech2021a} for the active model estimation problem. 
The main drawback is that $H_0$ depends linearly on the hitting time $\tilde{D}$ and not the diameter $D$ of the MDPs. 
As the hitting time is often strictly larger than the diameter~\citep{UCRL2,Tarbouriech2021a}, this dependency on $\tilde{D}$ is sub-optimal. 
On the other hand, AOMultiRL is substantially less computationally expensive than GOSPRL since there is no shortest-path policy computation involved.

\section{The generality of $\lambda$-separability notion}
\label{appendix:Example}
In this section, we show that the general separation notion in Definition~\ref{def:lambdaSeparability} defines a broader class of online multi-task RL problems that extends the entropy-based separation assumption in the latent MDPs setting~\citep{LatentMDPs2021}. We start by restating the entropy-based separation condition of~\cite{LatentMDPs2021}:

\begin{definition}
Let $\Pi$ denote the class of all history-dependent and possibly non-Markovian policies, and let $\tau \sim (m, \pi)$ be a trajectory of length $H$ sampled from MDP $m$ by a policy $\pi \in \Pi$. The set $\mathcal{M}$ is well-separated if the following condition holds:
\begin{align}
\forall m, m' \in \mathcal{M}, m' \neq m, \pi \in \Pi,
 \Pr_{\tau \sim (m, \pi)}\left( \frac{\Pr_{m', \pi}(\tau)}{\Pr_{m, \pi}(\tau)} > (\epsilon_p/M)^{c_1}\right) < (\epsilon_p / M)^{c_2},
\end{align}
where $\epsilon_p \in (0,1)$ is a target failure probability, $c_1 \geq 4, c_2 \geq 4$ are universal constants and $\Pr_{m,\pi}(\tau)$ is the probability that $\tau$ is realized when running policy $\pi$ on model $m$.
\label{def:entropySeparation}
\end{definition}

\begin{figure}[bp]
  \centering
  \begin{tikzpicture}
  [
  node distance=1.3cm,on grid,>={Stealth[round]},auto,
  pre/.style={<-,shorten <=1pt,>={Stealth[round]},semithick},
  twoway/.style={<->, dashed},
  inner sep=2mm,
  place/.style={circle,draw=blue!50,thick}
  ]
  \node [place] (a1) {1};
  \node [place] (a2)[below=of a1, xshift = -12mm, yshift = -10mm] {2}
  edge [pre] node {$\lambda$} (a1);
  \node [place] (a3)[below=of a1, xshift = 12mm, yshift = -10mm] {3}
  edge [pre] node[swap] {$1 - \lambda$} (a1)
  edge [twoway] node {} (a2);
  \node [place] (b1)[right=of a1, xshift = 25mm] {1}
  ;
  \node [place] (b2)[below=of b1, xshift = -12mm, yshift = -10mm] {2}
  edge [pre] node {$\lambda/2$} (b1);
  \node [place] (b3)[below=of b1, xshift = 12mm, yshift = -10mm] {3}
  edge [pre] node[swap] {$1 - \lambda/2$} (b1)
  edge [twoway] node {} (b2);
  \end{tikzpicture}
  \caption{An instance of $\lambda$-separable LMDPs where Definition~\ref{def:entropySeparation} does not apply}
  \label{fig:alphabeta}
  \end{figure}
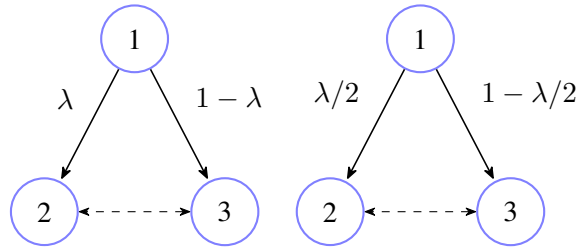

The following lemma constructs a set $\mathcal{M}$ of just two models that satisfy the $\lambda$-separability condition but not the entropy-based separation condition.

\begin{mylemma} Given any $\lambda \in (0,1), \epsilon_p \in (0,1), H > 0$ and any constants $c_1, c_2 \geq 4$, there exists a set of MDPs $\mathcal{M} = \{m_1, m_2\}$ with horizon $H$ that is $\lambda$-separable but is not well-separated in the sense of Definition~\ref{def:entropySeparation}.
\end{mylemma}
\begin{proof}
Consider the set $\mathcal{M}$ with $M = 2, \mathcal{S} = \{s^1, s^2, s^3\}, \mathcal{A} = \{a^1, a^2\}$ in Figure~\ref{fig:alphabeta}. Both $m_1$ and $m_2$ have the same transition functions in all state-action pairs except for $(s^1, a^1)$: 
\begin{align*}
& \P_1(s^2 \mid s^1, a^1) = \lambda \\
& \P_1(s^3 \mid s^1, a^1) = 1 - \lambda \\
& \P_2(s^2 \mid s^1, a^1) = \lambda / 2 \\
& \P_2(s^3 \mid s^1, a^1) = 1 - \lambda / 2. \\
\end{align*}
It follows that the $\ell_1$ distance between $P_1(s^1, a^1)$ and $P_2(s^1, a^1)$ is
\begin{align*}
\norm{P_1(s^1, a^1) - P_2(s^1, a^1)} &= \norm{P_1(s^2 \mid s^1, a^1) - P_2(s^2 \mid s^1, a^1 )} + \norm{P_1(s^3 \mid s^1, a^1) - P_2(s^3 \mid s^1, a^1)} \\
&= 2(\lambda - \lambda/2) \\
&= \lambda.
\end{align*}
As a result, this set $\mathcal{M}$ is $\lambda$-separable. However, any deterministic policy that takes action $a_2$ in $s_1$ and an arbitrary action in $s_2$ and $s_3$ will induce the same Markov chain on two MDP models. Thus, the entropy-based separation definition does not apply. An example of such a policy is shown below.

Consider running the following deterministic policy on model $m_1$:

\begin{align*}
\pi(s^1) &= a^2 \\
\pi(s^2) &= a^1 \\
\pi(s^3) &= a^1.
\end{align*}

Consider an arbitrary trajectory $\tau$.
The probability that this trajectory is realized with respect to both models is
\begin{align}
\Pr_{m_1, \pi}(\tau) &= \prod_{t=1}^{H}P_1(s_{t+1} \mid s_t, a_t )\\
&= \prod_{t=1}^{H}P_1(s_{t+1} \mid s_t, \pi(s_t) )\\
&= \prod_{t=1}^{H}P_2(s_{t+1} \mid s_t, a_t ) \qquad \text{since } (s_t, \pi(s_t)) \neq (s^1, a^1)\\
&= \Pr_{m_2, \pi}(\tau).
\end{align}

As a result, for all $\tau$, 
\begin{align}
  \frac{\Pr_{m_2, \pi}(\tau)}{\Pr_{m_1, \pi}(\tau)} = 1,  
\end{align}
which implies that 
\begin{align}
\Pr_{\tau \sim m_1, \pi}\left(\frac{\Pr_{m_2, \pi}(\tau)}{\Pr_{m_1, \pi}(\tau)} > (\epsilon_p / M)^{c_1}\right) = \Pr_{\tau \sim m_1, \pi}(1 > (\epsilon_p / M)^{c_1}) = 1,
\end{align}
which is larger than $\left(\epsilon_p / M\right)^{c_2}$.
\end{proof}

\section{Proofs of the lower bounds}
\label{sec:proofsLowerBounds}
\begin{figure}
	\centering
	\begin{tikzpicture}[->,>= stealth,shorten >=2 pt, line width =0.5 pt , node distance = 3 cm]
		\node [circle, draw] (zero) {0};
		\node [label=below:{+1},  circle , draw] (one) [ right of = zero] {1};
		\node [circle, draw] (two) [left of = zero] {2};
		\path (zero.north) edge [dashed, bend left = 100] node [above] {$\delta+\Delta$} (one);
		\path (zero) edge [ bend left = 50] node [above] {} (one) ;
		\path (zero) edge [ bend right = 50] node [below] {$\delta$} (one) ;
		\path (one) 
			edge [ bend right = 0] node [above] {$\delta$} (zero)
			edge [loop right] node {$1-\delta$} (one)
			;
		\path  (two) edge [dashed, bend left = 100] node [above] {$1/2+\frac{\lambda}{2}$}  (zero);
		\path  (two) edge [bend right = 50] node [above] {$1/2$} (zero);
		\path (two) edge [bend left = 50] node [below] {$1/2$} (zero);	
	\end{tikzpicture}
	\caption{A non-communicating 2-JAO MDP. There are no rewards at states $0$ and $2$, while state $1$ has reward $+1$. We set $\Delta = \Theta(\sqrt{\frac{SA}{HD}})$. The dashed arrows indicate the unique actions with highest transition probabilities on the left and right parts of the MDP. No actions take state $0$ to state $2$, making this MDP non-communicating.}
  \label{fig:noncommunicatingtwojaomdp}
\end{figure}
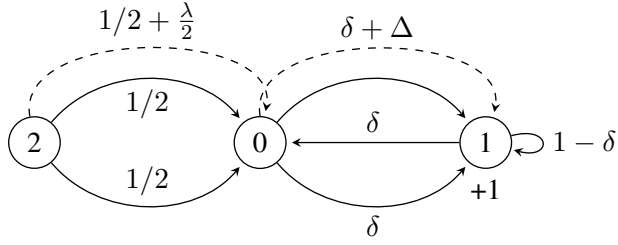

\LemmaMinimaxLowerBound*
\begin{proof}
We construct $\mathcal{M}$ in the following way: each MDP in $\mathcal{M}$ is a JAO MDP~\citep{UCRL2} of two states and $SA$ actions and diameter $D' = D/4$. The translation from this JAO MDP to an MDP with $S$ states, $A$ actions and diameter $D$ is straightforward~\citep{UCRL2}. 
State $1$ has reward $+1$ while state $0$ has no reward. In state $0$, for all actions the probability of transitioning to state $1$ is $\delta$ except for one best action where this probability is $\delta + \lambda/2$. Every MDP in $\mathcal{M}$ has a unique best action: for $i = 1, \dots, SA$, the $i\nth$ action is the best action in the MDP $m_i$. The starting state is always $s_1 = 0$.

We consider a learner who knows all the parameters of models in $\mathcal{M}$, except the identity of the task $m^k$ given in episode $k$. We employ the following information-theoretic argument from~\citet{Mao2021b}: when the task $m^k$ in episode $k$ is chosen uniformly at random from $\mathcal{M}$, no useful information from the previous episodes can help the learner identify the best action in $m^k$. 
This is true since all the information in the previous episodes is samples from the MDPs in $\mathcal{M}$, which provide no further information than the parameters of the models in $\mathcal{M}$. Since $\mathcal{M} = SA$, all actions (from state $0$) are equally probable to be the best action in $m^k$.  Therefore, the learner is forced to learn $m^k$ from scratch. It follows that the total regret of the learner is the sum of the one-episode-learning regrets in every episode:
\begin{align*}
  \mathrm{Regret}(K) = \sum_{k=1}^K R^k,
\end{align*}
where $R^k = V_1^*(s_1) - V_1^{\pi_k}(s_1)$ is the one-episode-learning regret in episode $k$. The one-episode-learning is equivalent to the learning in the undiscounted setting with horizon $H$. Applying the lower bound result for the undiscounted setting in~\citet[Theorem 5]{UCRL2} obtains that for all $\pi_k$,
\begin{align*}
  \rho^*H - \E_{m^k \sim \mathcal{M}}V_1^{\pi_k}(s_1) \geq \Omega(\sqrt{DSAH}),
\end{align*}
where $\rho^* = \frac{\delta+\lambda/2}{2\delta + \lambda/2}$ is the average reward of the optimal policy~\citep{UCRL2}. Note since only state $1$ has reward $+1$, $\rho^*$ is also the stationary probability that the optimal learner is at state $1$.

Next, we show that for all $H \geq 2$ and $m^k \in \mathcal{M}$, it holds that $|V_1^* - \rho^*H| \leq \frac{D}{2}$. The optimal policy on all $m^k$ induces a Markov chain between two states with transition matrix 
\begin{align*}
  \begin{bmatrix}
    1-\delta-\lambda/2 & \delta + \lambda/2 \\
    \delta & 1 - \delta
  \end{bmatrix}.
\end{align*}

Let $\P_{m^k}(s_t = 1 \mid s_1 = 0)$ be the probability that the Markov chain is in state $1$ after $t$ time steps with the initial state $s_1 = 0$.
Let $\Delta_t = \P_{m^k}(s_t = 1 \mid s_1 = 0) - \rho^*$. Obviously, $\Delta_1 = -\rho^*$. By~\citet[Equation 1.8]{MarkovChainAndMixingTime2008}, we have $\Delta_t = (1-2\delta-\lambda/2)^{t-1}\Delta_1$. It follows that, for the optimal policy, 
\begin{align}
  V_1^*(s_1) &= \sum_{t=1}^H \P_{m^k}(s_t = 1 \mid s_1 = 0) \\
             &= \sum_{t=1}^H (\Delta_t + \rho^*) \\
             &= \rho^*H + \sum_{t=1}^H \Delta_t \\
             &= \rho^*H + \sum_{t=1}^H (1-2\delta-\lambda/2)^{t-1}\Delta_1 \\
             &= \rho^*H + \Delta_1 \frac{1 - (1-2\delta-\lambda/2)^H}{2\delta + \lambda/2}.
\end{align}

Hence,
\begin{align} 
  |V_1^*(s_1) - \rho^*H| &= \left| \Delta_1 \frac{1 - (1-2\delta-\lambda/2)^H}{2\delta + \lambda/2} \right| \\
  &\leq \left| \frac{\Delta_1}{2\delta + \lambda/2} \right| \\
  &= \frac{\rho^*}{2\delta + \lambda/2} \\
  &\leq \frac{1}{2\delta + \lambda / 2} \\
  &\leq \frac{1}{2\delta} \\
  &= \frac{D}{2},
  \label{eq:boundV1rhoH}
\end{align}
where the last equality follows from $\delta = \frac{D}{4}$.

For any $H \geq DSA$ and $S, A \geq 2$, we have $\sqrt{HDSA} \geq DSA \geq 4D$, and hence $\sqrt{HDSA} - \frac{D}{2} \geq \frac{\sqrt{HDSA}}{2}$. We conclude that 
\begin{align*}
  \E[\mathrm{Regret}(K)] &= \sum_{k=1}^K \E[R^k] \\ 
  &= \sum_{k=1}^K \E[V^*_1 - V_1^{\pi_k}](s_1) \\
  &\geq \sum_{k=1}^K (\rho^*H - \frac{D}{2} - V_1^{\pi_k}(s_1)) \\
  &= \Omega(K\sqrt{DHSA}).
\end{align*}
\end{proof}

The upper bound of UCRL2 can be proved similarly: Theorem 2 in~\citet{UCRL2} states that for any $p \in (0, 1)$, by running UCRL2 with failure parameter $p$, we obtain that for any initial state $s_1$ and any $H > 1$, with probability at least $1-p$,
\begin{align}
  \rho^*H - \sum_{h=1}^H r_h \leq O\left(DS\sqrt{AH\ln{\frac{H}{p}}}\right).
\end{align}

Setting $p = \frac{1}{H}$ and trivially bound the regret in the failure cases by $H$ to obtain

\begin{align}
  \rho^*H - E[\sum_{h=1}^H r_h] &\leq O\left(DS\sqrt{AH\ln{H^2}}\right) + \frac{1}{H} \times H \\
  &=  O\left(DS\sqrt{AH\ln{H}}\right).
\end{align}

This bound holds across all episodes, hence the total regret bound with respect to $\rho^*H$ is $O\left(KDS\sqrt{AH\ln{H}}\right)$. Combining this with the fact that $V_1^*(s_1) \leq \rho^*H + \frac{D}{2}$, we obtain the upper bound.

\LemmaPACLowerBoundQCoins*
Before showing the proof of Lemma~\ref{lemma:PACLowerBoundQCoins}, we consider the following auxiliary problem: Suppose we are given three constants $\delta, \lambda, \epsilon \in (0, \frac{1}{4}]$ and a set of $2Q$ coins. The coins are arranged into a $Q \times 2$ table of $Q$ rows and $2$ columns so that each cell contains exactly one coin. The rows are indexed from $1$ to $Q$ and the columns are indexed from $1$ to $2$. In the first column, all coins are fair except for one coin at row $\theta$ which is biased with probability of heads equal to  $\frac{1}{2}+\lambda$. In the second column, all coins have probability of heads equal to $\delta$ except for the coin at row $\theta$ which has probability of heads $\delta + \epsilon$. In this setting, row $\theta$ is a special row that contains the most biased coins in the two columns. The objective is to find this special row $\theta$ after at most $H$ coin flips, where $H > 0$ is a constant representing a fixed budget. Note that if we ignore the second column, then this problem is reduced to the well-known problem of identifying one biased coin in a collection of $Q$-coins~\citep{Tulsiani2014L6}.

Let \budgetfix{$N_1, N_2$} be the number of flips an algorithm performs on the first and second column, respectively. \budgetfix{For a fixed global budget $H$,} after \budgetfix{$\tau = N_1 + N_2 \leq H$} coin flips, the algorithm recommends $\hat{\theta}$ as its prediction for $\theta$. \budgetfix{Note that $\tau$ is a random stopping time which can depend on any information the algorithm observes up to time $\tau$.} 
Let $X_t$ be the random variable for the outcome of $t\nth$ flip, and \budgetfix{$X_1^\tau = (X_1, X_2, \dots,  X_\tau)$} be the sequence of outcomes after \budgetfix{$\tau$} flips. For $j \in [Q]$, let $\P_j$ denote the probability measure induced by $\mathrm{Alg}$ corresponding to the case when $\theta = j$.
We first show that if the algorithm fails to flip the coins sufficiently many times in both columns, then for some $\theta$ the probability of failure is at least $\frac{1}{2}$.


\begin{mylemma}
  Let $Q \geq 12, C_1 = 40$ and $C_2 = 64$. For any algorithm $\mathrm{Alg}$, if 

  \budgetfix{
  \begin{align*}
    N_1 \leq T_1 := \frac{Q}{4 C_1 \lambda^2} 
    \quad \text{and} \quad 
    N_2 \leq T_2 := \frac{Q(\delta + \epsilon)}{4 C_2 \epsilon^2},
  \end{align*}
}
then there exists a set $J \subseteq [Q]$ with $|J| \geq \frac{Q}{6}$ such that 
  \begin{align*}
    \forall j \in J,~\P_{j}[\hat{\theta}=j] \leq \frac{1}{2}.
  \end{align*}
  \label{lemma:generalized-tulsiani}
\end{mylemma}

The proof uses a reasonably well-known reverse Pinsker inequality \citep[Equation 10]{sason2015reverse}:
\begin{quote}
Let $P$ and $Q$ be probability measures over a common discrete set. Then
\begin{align}
\infdiv{P}{Q} \leq \frac{4 \log_2 e}{\min_x Q(x)} \cdot D_{TV}(P \pipes Q)^2 . \label{eqn:reverse-pinsker}
\end{align}
\end{quote}
where $D_{TV}$ is the total variation distance. In the particular case where $P$ and $Q$ are Bernoulli distributions with success probabilities $p$ and $q \leq \frac{1}{2}$ respectively, we get
\begin{align}
\infdiv{P}{Q}
\leq \frac{4 \log_2 e}{q} \cdot (p - q)^2. \label{eqn:reverse-pinsker-pq}
\end{align}

\begin{proof}(of Lemma~\ref{lemma:generalized-tulsiani})
  As reasoned in the proof for the lower bound of multi-armed bandits~\citep{EXP3Auer2002b}, we can assume that $\mathrm{Alg}$ is deterministic\footnote{Deterministic conditional on the random history}.
  Our proof closely follows the main steps in the proof of~\citet{Tulsiani2014L6} for the setting where there is only one column. We will lower bound the probability of mistake of  $\mathrm{Alg}$ based on its behavior on a hypothetical instance where $\lambda = \epsilon = 0$. 
  
  \budgetfix{
    To account for algorithms which do not exhaust both budgets $T_1$ and $T_2$, we introduce two ``dummy coins'' by adding a zero’th row with two identical coins, solely for the analysis. These two coins have the same mean of 1 under all $Q$ models and hence flipping either of them provides no information. An algorithm which wishes to stop in a round $\tau < H$ will simply flip any dummy coin in the remaining rounds $\tau+1, \tau + 2, \ldots, H$. This way, we have the convenient option of always working with a sequence of outcomes $X_1^H$ in the analysis.
    }

  Let $\P_0$ and $\E_0$ denote the probability and expectation over $X_1^H$ taken on the hypothetical instance with $\lambda = \epsilon = 0$, respectively. Let \budgetfix{$a_t = (a_{t, 0}, a_{t, 1}) \in \{0, 1, \ldots, Q\} \times \{1, 2\}$} be the coin that the algorithm flips in step $t$. Let $x_t \in \{0, 1\}$ denote the outcome of $a_t$ where $0$ is tails and $1$ is heads. 
 
 The number of flips the coin in row $i$, column $k$ is 
  \begin{align*}
    N_{i, k} = \sum_{t=1}^T \I{a_t = (i, k)}.
  \end{align*}

  \budgetfix{
By the earlier definition of $N_k$ for $k \in \{1, 2\}$, we have
 \begin{align*}
  N_1 = \sum_{i=1}^Q N_{i, 1}, \\
  N_2 = \sum_{i=1}^Q N_{i, 2}.
 \end{align*}
}

 We define
 \begin{align*}
 J_1 := \left\{ i \in [Q] \colon \left( \E_0 [ N_{i,1} ] \leq \frac{4 T_1}{Q} \right)
                                               \wedge \left( \E_0 [ N_{i,2} ] \leq \frac{4 T_2}{Q} \right)
           \right\} .
 \end{align*}
 Clearly, at most $\frac{Q}{4}$ rows $i$ satisfy $\E_0 [ N_{i,1} ] > \frac{4 T_1}{Q}$ and, similarly, at most $\frac{Q}{4}$ rows $i$ satisfy $\E_0 [ N_{i,2} ] > \frac{4 T_2}{Q}$. Therefore, $|J_1| \geq Q - 2 \cdot \frac{Q}{4} = \frac{Q}{2}$.
 
 We also define
 \begin{align*}
 J_2 := \left\{ i \in [Q] \colon \P_0 (\hat{\theta} = i) \leq \frac{3}{Q} \right\} .
 \end{align*}
 As at most $\frac{Q}{3}$ arms $i$ can satisfy $\P_0(\hat{\theta} = i) > \frac{3}{Q}$, it holds that $|J_2| \geq \frac{2 Q}{3}$.
 
 Consequently, defining $J := J_1 \cap J_2$, we have $|J| \geq \frac{Q}{6}$.

 For any $j \in J$, we have 
  \begin{align}
    \left|\P_j[c^* = j] - \P_0[c^*=j]\right| &= \left|\E_j[\I{c^* = j}] - \E_0[\I{c^*=j}]\right| \\
    &\leq \frac{1}{2}\norm{\P_0(\budgetfix{X_1^H}) - \P_j(\budgetfix{X_1^H})}_1 \\
    &\leq \frac{1}{2}\sqrt{2\ln{2}\infdiv{P_0(\budgetfix{X_1^H})}{\P_j(\budgetfix{X_1^H})}},
    \label{eq:substractPboundKL}
  \end{align}
 where the first inequality follows from~\citet[Equation 28]{EXP3Auer2002b} since the final output $c^*$ is a function of the outcomes $X_{1}^H$, and the last inequality is Pinsker inequality. 

 Since $\mathrm{Alg}$ is deterministic, the flip $a_t$ at step $t$ is fully determined given the previous outcomes $x_1^{t-1}$. Applying the chain rule for KL-divergences~\citep[Theorem 2.5.3]{CoverAndThomas2006} we obtain
 \begin{align*}
   \infdiv{P_0(\budgetfix{X_1^H})}{\P_j(\budgetfix{X_1^H})} &= \sum_{t=1}^H\sum_{x_{1}^{t-1}}\P_0[x_{1:t-1}] \infdiv{\P_0[x_t]}{\P_j[x_t] \mid x_{1}^{t-1}}.
 \end{align*}

 Note that $x_t$ is the result of a single coin flip. When $a_{t, 0} \neq j$, the KL-divergence is zero since the two instances have the identical coins on both columns. When $a_{t,0} = j$, the KL-divergence is either $B_1 = \infdiv{\frac{1}{2}}{\frac{1}{2}+\lambda}$ or $B_2 = \infdiv{\delta}{\delta+\epsilon}$, depending on whether $a_{t, 1}=1$ or $a_{t,1} = 2$, respectively. It follows that
 \begin{align*}
   \infdiv{P_0(\budgetfix{X_1^H})}{\P_j(\budgetfix{X_1^H})} &= \sum_{t=1}^{\budgetfix{H}} \sum_{x_{1:t-1}}\P_0[x_{1:t-1}]\left( \I{a_t = (j, 1)}B_1 + \I{a_t = (j, 2)}B_2 \right) \\
   &= \E_0[N_{j, 1}]B_1 + \E_0[N_{j, 2}]B_2 \\
   &\leq \frac{4T_1}{Q}B_1 + \frac{4T_2}{Q}B_2 \\
   &\leq \frac{B_1}{C_1\lambda^2} + \frac{(\delta + \epsilon)B_2}{C_2\epsilon^2}
 \end{align*}

Since $\lambda \leq \frac{1}{4}$ and $\delta + \epsilon \leq \frac{1}{2}$, we can bound $B_1 \leq \frac{5\lambda^2}{2\ln{2}}$~\citep{Tulsiani2014L6} and $B_2 \leq \frac{4\log_2(e)\epsilon^2}{\delta + \epsilon}$. Consequently,
\begin{align*}
 \infdiv{P_0(\budgetfix{X_1^H})}{\P_j(\budgetfix{X_1^H})} &\leq \frac{5}{(2\ln{2})C_1} + \frac{4\log_2(e)}{C_2}
\end{align*}

Plugging this into Equation~\ref{eq:substractPboundKL} and applying $Q \geq 12$, we obtain 
\begin{align*}
 \P_j[\hat{\theta}=j] &\leq \P_0[\hat{\theta}=j] + \frac{1}{2}\sqrt{2\ln{2}\left( \frac{5}{(2\ln2)C_1} + \frac{4\log_2(e)}{C_2} \right)} \\
 &= \frac{3}{Q} + \frac{1}{2}\sqrt{\frac{5}{C_1} + \frac{8}{C_2}} \\
 &\leq \frac{3}{12} + \frac{1}{2}\sqrt{\frac{5}{40}+\frac{8}{64}} \\
 &= \frac{1}{2}.
\end{align*}
\end{proof}

The next result shows that if $\epsilon$ is sufficiently small, then any algorithm has to flip the coins in the first column sufficiently many times; otherwise the probability of failure is at least $\frac{1}{2}$.

\begin{mycorollary} Let $Q, C_1$ and $C_2$ be the constants defined in Lemma~\ref{lemma:generalized-tulsiani}. Let $H > 0$ be the budget for the number of flips on both columns. If $\epsilon = \frac{1}{20}\sqrt{\frac{Q\delta}{H}}$, then for any algorithm $\mathrm{Alg}$, if 
  \begin{align*}
    N_1 \leq \frac{Q}{4C_1\lambda^2},
  \end{align*}
  then there exists a set $J \subseteq [Q]$ with $|J| \geq \frac{Q}{6}$ such that 
  \begin{align*}
    \forall j \in J, \P_{j}[\hat{\theta}=j] \leq \frac{1}{2}.
  \end{align*}
  \label{corollary:generalized-tulsiani}
\end{mycorollary}
\begin{proof}
  We will show that when $\epsilon = \frac{1}{20}\sqrt{\frac{Q\delta}{H}}$, the inequality \budgetfix{$N_2 \leq T_2 = \frac{Q(\delta+\epsilon)}{4C_2\epsilon^2}$} holds trivially for any \budgetfix{$N_2 \leq H$} (recall that $H$ is the fixed budget for the total number of coin flips). The result then follows directly from Lemma~\ref{lemma:generalized-tulsiani}. We have 
  \begin{align*}
    \budgetfix{T_2 = } \frac{Q(\delta+\epsilon)}{4C_2\epsilon^2} &\geq \frac{Q\delta}{4C_2\epsilon^2} \\
    &= \frac{Q\delta}{256\epsilon^2} \qquad\text{ since } C_2 = 64 \\
    &= \frac{400}{256}H \\
    &> H \\
    &\geq \budgetfix{N_2},
  \end{align*}
  which implies that \budgetfix{$N_2 \leq T_2$} always holds for any \budgetfix{$N_2 \leq H$.}

\end{proof}

We are now ready to prove Lemma~\ref{lemma:PACLowerBoundQCoins}.

\begin{proof}(of Lemma~\ref{lemma:PACLowerBoundQCoins}) We construct $\mathcal{M}$ as the set of $\frac{SA}{12}$ 2-JAO MDPs in Figure~\ref{fig:jaomdp} (right). Each MDP has a left part and a right part, where each part is a JAO MDP. 
The left part of the MDP $m_i$ consists of two states $\{0, 2\}$ and $\frac{SA}{12}$ actions numbered from $1$ to $\frac{SA}{12}$, where all actions from state $0$ transition to state $2$ with probability of $\frac{1}{2}$ or stay at state $0$ with probability $\frac{1}{2}$, except for the $i\nth$ action that transitions to state $2$ with probability $\frac{1}{2} + \frac{\lambda}{2}$ and stays at state $0$ with probability $\frac{1}{2}-\frac{\lambda}{2}$. The right part of the $i\nth$ MDP consists of two states $\{0, 1\}$ and also $\frac{SA}{12}$ actions numbered from $1$ to $\frac{SA}{12}$, where all actions from state $0$ transition to state $1$ with probability of $\delta = \frac{4}{D} \leq \frac{1}{4}$  or stays at state $0$ with probability $1-\delta$, except for the $i\nth$ action that transitions to state $2$ with probability $\delta + \Delta$ and stays at state $0$ with probability $1-\delta-\Delta$. 
  We set $\Delta = \frac{1}{20}\left(\sqrt{\frac{SA}{3HD}}\right)$.
We will show the conversion from these 2-JAO MDPs to MDPs with $S$ states and $A$ actions later.

Since each model in $\mathcal{M}$ has a distinct index for the actions on both parts that transitions from $0$ to $1$ and $2$ with probability higher than any other actions, identifying a model in $\mathcal{M}$ is equivalent to identifying this distinct action. 
Each action on both parts can be seen as a (possibly biased) coin, where the probability of getting tails is equal to the probability of ending up in state $0$ when the action is taken. Thus, the problem of identifying this distinct action index reduces to the above auxiliary problem of identifying the row of the most biased coins, where taking an action from state $0$ is equivalent to flipping a coin, $Q = \frac{SA}{12} \geq 12, \epsilon = \Delta$ and $\lambda$ is replaced by $\lambda/2$.
Corollary~\ref{corollary:generalized-tulsiani} states that for every algorithm, if the number of coin flips on the first column is less than $\frac{SA}{480\lambda^2}$, then there exists a set of size at least $\frac{SA}{72}$ positions of the row with the most biased coins such that the algorithm fails to find the biased coin with probability at least $\frac{1}{2}$. Correspondingly, for any model classification algorithm, if the number of state-transition samples from state 0 towards state 2 (i.e. the first column) is less than $\frac{SA}{480\lambda^2}$ then the algorithm fails to identify the model for at least $\frac{SA}{72}$ MDPs in $\mathcal{M}$.

Finally, we show the conversion from the 2-JAO MDP to an MDP with $S$ states and $A$ actions. The conversion is almost identical to that of~\citet{UCRL2}, which starts with an \textit{atomic} 2-JAO MDP of three states and $A' = \frac{A}{2}$ actions and builds an $A'$-ary tree from there. 
Assuming $A'$ is an even positive number, each part of the atomic 2-JAO MDP has $\frac{A'}{2}$ actions. We make $\frac{S}{3}$ copies of these atomic 2-JAO MDPs, where only one of them has the best action on the right part. Arranging $\frac{S}{3}$ copies of these atomic 2-JAO MDPs and connecting their states $0$ by $A-A'$ connections, we obtain an $A'$-ary tree which represents a composite MDP with at most $S$ states, $A$ actions and diameter $D$. The transitions of the $A-A'$ actions on the tree are defined identically to that of~\citet{UCRL2}: self-loops for states $1$ and $2$, deterministic connections to the state $0$ of other nodes on the tree for state $0$. By having $\delta = \frac{4}{D}$ in each atomic 2-JAO MDP, the diameter of this composite MDP is at most $\frac{2}{\delta} + \log_{A'}{\frac{S}{3}} \leq D$. This composite MDP is harder to explore and learn than the 2-JAO MDP with three states and $\frac{SA}{6}$ actions, and hence all the lower bound results apply.

\end{proof}

\CorollaryPACLowerBoundSteps*
\begin{proof}  
  As argued in Section~\ref{sec:lowerbounds}, we can apply the sample complexity of the classification algorithm onto that of the clustering algorithm.
Using the same set $\mathcal{M}$ of 2-JAO MDPs constructed in the proof of Lemma~\ref{lemma:PACLowerBoundQCoins}, for any given MDP $\mathcal{M}$, any PAC classification learner has to be in state $0$ and takes at least $Z = \Omega(\frac{SA}{\lambda^2})$ actions from state $0$ to state $2$. If the learner stays at state $0$, then it can take the next action from $0$ to $2$ in the next time step. However, if the learner transitions to state $2$, then it has to wait until it gets back to state $0$ to take the next action. Let $Z_2$ denote the number of times the learner ends up in state $2$ after taking $Z$ actions on the left part from state $0$. Since every action from $0$ to $2$ has probability at least $\frac{1}{2}$ of ending up in state $2$, we have 
  \begin{align}
    \E[Z_2 \mid Z] \geq \frac{Z}{2}.
  \end{align}
   Since every action from state $2$ transitions to state $0$ with the same probability of $\delta = \Theta(\frac{1}{D})$, every time the learner is in state $2$, the expected number of time steps it needs to get back to state $0$ is $\Theta(\frac{1}{\delta}) = \Theta(D)$. Hence, the expected number of time steps the learner needs to get back to state $0$ after $Z_2$ times being in state $2$ is $\Theta(Z_2D)$. We conclude that for any PAC learner, the expected number of exploration steps needed to identify the model with probability of correct at least $\frac{1}{2}$ is at least 
   \begin{align}
    \E[Z + Z_2D] \geq \Omega(ZD) = \Omega\left(\frac{DSA}{\lambda^2}\right).
   \end{align}
   Next, we lower bound the expected regret of the same algorithm. Let $H_0$ be the number of time steps the algorithm spends on the left part and $H_1$ on the right part of each model in $\mathcal{M}$. Note that $H_0$ and $H_1$ are random variables.
  Recall that the right part of each MDP in $\mathcal{M}$ resembles the JAO MDP in the minimax lower bound proof in Lemma~\ref{lemma:minimaxLowerBound}, hence we can apply the regret formula of the JAO MDP for 2-JAO MDP and obtain that the regret in each episode is of the same order as
  \begin{align}
    \mathrm{Regret} &= \rho^*H - \E[\sum_{h=1}^H r(s_h, a_h)] \\
    &= \rho^*E[H_0 + H_1] - \E[\E[\sum_{h=1}^{H_0} r(s_h, a_h)] + \E[\sum_{h=H_0+1}^{H} r(s_h, a_h)] \mid H_0, H_1] \\
    &= \rho^*E[H_0 + H_1] - \E[\E[\sum_{h=H_0+1}^{H} r(s_h, a_h) \mid H_0, H_1]] \\
    &= \rho^*E[H_0] + E\left[\left(\rho^*H_1 - \E[\sum_{h=H_0+1}^{H} r(s_h, a_h)] \right)\mid H_1\right] \\
    &\geq \Omega\left(\rho^*E[H_0]\right) - \frac{D}{2} \\
    &= \Omega\left(\frac{DSA}{\lambda^2}\right),
  \end{align}
  where 
  \begin{itemize}
    \item the second equality follows from $H = H_0 + H_1$,
    \item the third equality follows from the fact that the $H_0$ time steps spent on the left part of the MDP returns no rewards,
    \item the fourth equality follows from the linearity of expectation,
    \item the inequality follows from $H_1 = H - H_0$ and~\eqref{eq:boundV1rhoH},
    \item the last equality follows from $\rho^* = \frac{\delta + \Delta}{2\delta + \Delta} \geq \frac{1}{2}$ for all $\delta, \Delta > 0$ and $E[H_0] \geq \Omega\left(\frac{DSA}{\lambda^2}\right)$.
  \end{itemize}

    We conclude that the expected regret over $K$ episodes is at least 
  \begin{align*}
    \Omega(\E[KH_0]) = \Omega\left(\frac{KDSA}{\lambda^2}\right).
  \end{align*}
\end{proof}

\section{Proofs of the upper bounds}
\label{appendix:proofs}

First, we state the following concentration inequality for vector-valued random variables by~\cite{weissman2003inequalities}. 

\begin{mylemma}[\cite{weissman2003inequalities}] Let $P$ be a probability distribution on the set $\mathcal{S} = \{1, \dots, S\}$. Let $\mathcal{X}^N$ be a set of $N$ i.i.d samples drawn from $P$. Then, for all $\epsilon > 0$:
\[
\Pr(\norm{P - \hat{P}_{\mathcal{X}^N}} \geq \epsilon) \leq (2^S - 2)e^{- N\epsilon^2/2}.
\]
\label{lemma:weismann}
\end{mylemma}
Using Lemma~\ref{lemma:weismann}, we can show that $N = O(\frac{S}{\lambda^2})$ samples are sufficient for each $(s, a) \in \Gamma$ so that with high probability, the empirical means of the transition function $\hat{P}_{\mathcal{B}}(\cdot \mid s, a)$ are within $\lambda/8$ of their true values, measured in $\ell_1$ distance. 

\begin{mycorollary} Denote $p_1 \in (0, 1)$. If a state-action pair $(s,a)$ is visited at least 
  \begin{align}
    N = \frac{256}{\lambda^2}\max\{S, \ln(1/p_1)\}
  \end{align} times, then with probability at least $1 - p_1$,
\begin{align*}
\norm{P(s,a) - \hat{P}_{\mathcal{X}^N}(s,a)} \leq \lambda / 8.
\end{align*}
\label{corollary:fixN}
\end{mycorollary}
\begin{proof}
We simplify the bound in Lemma~\ref{lemma:weismann} as follows:
\[
\Pr(\norm{P - \hat{P}_{\mathcal{X}^N}} \geq \epsilon) \leq (2^S - 2)e^{- N\epsilon^2/2} \leq e^{S - N\epsilon^2/2}
\]
Next, we substitute $\epsilon = \lambda/8$ into the right hand side and solve the following inequality for $N$:
\[
\begin{split}
e^{S - N\lambda^2/128} \leq p_1
\end{split}
\]
to obtain $N \geq \frac{128}{\lambda^2}(S + \ln(1/p_1))$. Thus $N = \frac{256}{\lambda^2}\max\{S, \ln(1/p_1)\}$ satisfies this condition.
\end{proof}

Taking a union bound of the result in Corollary~\ref{corollary:fixN} over all state-action pairs in the set $\Gamma$ of all episodes from $1$ to $K$, we obtain Lemma~\ref{lemma:goodeEstimatorId}.

Next, we show the proof of Lemma~\ref{lemma:computeH0}. The proof strategy is similar to that of~\citet{UCRL2006,Sun2020}.

\LemmaComputeH*
\begin{proof}
    The history-dependent exploration policy in Algorithm~\ref{algo:ExploreID} visits an under-sampled state-action pair in $\Gamma^\alpha$ whenever possible; otherwise it starts a sequence of steps that would lead to such a state-action pair. In the latter case, denote the current state of the learner by $s$ and the number of steps needed to travel from $s'$ to an under-sampled state $s$ by $T(s', s)$. By Assumption~\ref{assumption:diameter} and using Markov inequality, we have
    \[
    \Pr(T(s', s) > 2\tilde{D}) \leq \frac{E[T(s',s)]}{2\tilde{D}} \leq \frac{\tilde{D}}{2\tilde{D}} = \frac{1}{2}.
    \]
    
    It follows that $\Pr(T(s', s) > 2\tilde{D}) \leq 1/2$. In other words, in every interval of $2\tilde{D}$ time steps, the probability of visiting an under-sampled state-action pair in $\Gamma^\alpha$ is at least $1/2$. Over such $n$ intervals, the expected number of such visits is lower bounded by $n/2$. Fix a $(s, a) \in \Gamma^\alpha$. Let $V_n$ denote number of visits to $(s, a) \in \Gamma^\alpha$ after $n$ intervals. Using a Chernoff bound for Poisson trials, we have
    \[
    \Pr(V_n \geq (1-\epsilon)n/2) \geq 1 - e^{-\epsilon^2n/4}
    \]
    for any $\epsilon \in (0, 1)$. Setting $\epsilon = 1 - 2N/m$ and solving
    \[
    e^{-(1-2N/n)^2n/4} \leq p_1
    \]
    for $n$, we obtain
    \begin{equation}
    n \geq 2(N + \ln(1/p_1)) + 2\sqrt{2N\ln(1/p_1) + (\ln(1/p_1))^2}.
    \label{eq:numberOfIntervals}
    \end{equation}
    
    By definition of $N$,
    \[
    \begin{split}
    2N\ln(1/p_1) + (\ln(1/p_1))^2 &\leq  (1 + \frac{512}{\lambda^2})\max\{S, \ln(1/p_1)\}^2 \\
    &\leq \left(\frac{256}{\lambda}\max\{S, \ln(1/p_1)\}\right)^2 \\
    &\leq N^2.
    \end{split}
    \]
    
    We also have $N \geq \ln(1/p)$. Overall, $n = 6N$ satisfies the condition in Equation~\ref{eq:numberOfIntervals}. Taking a union bound over all $(s, a) \in \Gamma^\alpha$ and noting that each interval has length $2\tilde{D}$ steps, the total number of identifying steps needed is $H_0 = 2\tilde{D}n|\Gamma^\alpha| = 12\tilde{D}|\Gamma^\alpha|N$. 
    \end{proof}

To prove Lemma~\ref{lemma:noIncorrectClusters}, we state the following auxiliary proposition and its corollary.
\begin{restatable}{myproposition}{PropAddSamplesToCluster} Suppose we are given a probability distribution $P$ over $\mathcal{S} = 1, \dots, S$, a constant $\epsilon > 0$ and two set of samples $\mathcal{X} = (X_1, \dots, X_{N_\mathcal{X}})$ and $\mathcal{Y} = (Y_1, \dots, Y_{N_\mathcal{Y}})$ drawn from $P$ such that $\norm{P - \hat{P}_{\mathcal{X}}} \leq \epsilon$ and $\norm{P - \hat{P}_{\mathcal{Y}}} \leq \epsilon$. Then,
\[
\norm{P - \hat{P}_{\mathcal{X} \cup \mathcal{Y}}} \leq \epsilon.
\]
\label{prop:addSamplesToCluster}
\end{restatable}
\begin{proof}
    Let $N_{\mathcal{X}}(s)$ and $N_{\mathcal{Y}}(s)$ denote the number of samples of $s \in [S]$ in $\mathcal{X}$ and $\mathcal{Y}$, respectively. We have:
    \begin{align}
    \norm{P - \hat{P}_{\mathcal{X} \cup \mathcal{Y}}} &= \sum_{s =1}^S |P(s) - \frac{N_{\mathcal{X}}(s) + N_{\mathcal{Y}}(s)}{N_\mathcal{X} + N_\mathcal{Y}}| \\
    &= \frac{1}{N_\mathcal{X} + N_\mathcal{Y}}\sum_{s=1}^S |N_\mathcal{X} P(s) - N_{\mathcal{X}}(s) + N_\mathcal{Y} P(s) - N_{\mathcal{Y}}(s)| \\
    &\leq \frac{1}{N_\mathcal{X} + N_\mathcal{Y}}\sum_{s=1}^S (|N_\mathcal{X} P(s) - N_{\mathcal{X}}(s)| + |N_\mathcal{Y} P(s) - N_{\mathcal{Y}}(s)|) \quad \text{(triangle inequality)}\\
    &= \frac{1}{N_\mathcal{X} + N_\mathcal{Y}}\left(N_\mathcal{X} \sum_{s=1}^S |P(s) - \frac{N_{\mathcal{X}}(s)}{N_\mathcal{X}}|\right) + \frac{1}{N_\mathcal{X} + N_\mathcal{Y}}\left(N_\mathcal{Y}\sum_{s=1}^S |P(s) - \frac{N_{\mathcal{Y}}(s)}{N_\mathcal{Y}}|\right) \\
    &= \frac{1}{N_\mathcal{X} + N_\mathcal{Y}}(N_\mathcal{X}\norm{P - \hat{P}_{\mathcal{X}}}_1 + N_\mathcal{Y}\norm{P - \hat{P}_{\mathcal{Y}}}) \\
    &\leq \epsilon
    \end{align}
    \end{proof}

\begin{mycorollary}
\label{cor:addSamplesToCluster}
Suppose we are given a probability distribution $P$ over $\mathcal{S} = 1, \dots, S$, a constant $\epsilon > 0$ and a finite number of set of samples $\mathcal{X}_1, \mathcal{X}_2, \dots, \mathcal{X}_t$ such that $\norm{P - \hat{P}_{\mathcal{X}_i}} \leq \epsilon$ for all $i = 1, 2, \dots, t$. Then,
\begin{align}
\norm{P - \hat{P}_{\cup_{i=1,\dots,t}\mathcal{X}_i}} \leq \epsilon.
\end{align}
\end{mycorollary}

\begin{proof}(Of Lemma~\ref{lemma:noIncorrectClusters})
The proof is by induction. The claim is trivially true for the first episode ($k = 1$). For an episode $k > 1$, assume that the outputs of the Algorithm~\ref{algo:IdentifyCluster} are correct until the beginning of this episode. We consider two cases:

\begin{itemize}
\item When the task $m_k$ has never been given to the learner before episode $k$.

Consider an arbitrary existing cluster $c$. Denote by $i \in [M]$ the identity of the model to which the samples in $c$ belong, $j \in [M]$ the identity of the task $m_k$, and $(s, a)$ in $\Gamma^\alpha_{i,j}$ a state-action pair that distinguishes these two models. Under the definition of $\Gamma^{\alpha}_{i,j}$, the result in Lemma~\ref{lemma:goodeEstimatorId} and the result in Corollary~\ref{cor:addSamplesToCluster}, the following three inequalities hold true:
\[
\begin{split}
\norm{[P_i - P_j](s, a)} &> \alpha \\
\norm{[P_j - \hat{P}_{\mathcal{B}_k}](s, a)} &\leq \lambda/8 \\
\norm{[P_i - \hat{P}_c](s, a)} &\leq \lambda/8. \\
\end{split}
\]

From here, we omit the $(s, a)$ and write $P$ for $P(s, a)$ when no confusion is possible. Applying the triangle inequality twice, we obtain:
\[
\begin{split}
\norm{\hat{P}_c - \hat{P}_{\mathcal{B}_k}} &\geq \norm{P_i - P_j} - (\norm{P_i - \hat{P}_c} + \norm{P_j - \hat{P}_{\mathcal{B}_k}}) \\
&> \alpha - (\lambda/8 + \lambda/8) \\
&= \delta.
\end{split}
\]

It follows that the \texttt{break} condition in Algorithm~\ref{algo:IdentifyCluster} is satisfied, and the correct value of 0 is returned. A new cluster is created containing only the samples generated by the new task $m_k$.

\item When the task $m_k$ has been given to the learner before episode $k$.

In this case, there exists a cluster $c'$ containing the samples generated from model $j$. Using a similar argument in the previous part, we have that whenever the iteration in Algorithm~\ref{algo:IdentifyCluster} reaches a cluster $c$ whose identity $i \neq j$, the \texttt{break} condition is true for at least one $(s, a) \in \Gamma^\alpha$, and the algorithm moves to the next cluster. When the iteration reaches cluster $c'$, for all $(s, a) \in \tilde{\Gamma}^\alpha$, we have:
\[
\begin{split}
\norm{\hat{P}_{\mathcal{B}_k} - \hat{P}_{c'}} &\leq \norm{\hat{P}_{\mathcal{B}_k} - P_j} + \norm{P_j - \hat{P}_{c'}} \\
&\leq \lambda/8 + \lambda/8 = \lambda/4 \\
&\leq \delta.
\end{split}
\]

Hence, the \texttt{break} condition is false for all $(s,a) \in \Gamma$, and thus the algorithm returns $\texttt{id} = c'$ as expected.
\end{itemize}

By induction, under event $\mathcal{E}_\Gamma$, Algorithm~\ref{algo:IdentifyCluster} always produces correct outputs throughout the $K$ episodes.
\end{proof}

We can now state the regret bound of Algorithm~\ref{algo:aomtrl} where the regret minimization algorithm in every episode is UCBVI-CH~\citep{Azar2017}. For each state-action pair $(s, a)$ in episode $k$, UCBVI-CH needs a bonus term defined as
\[
b_k(s, a) = 7H_1L_k\sqrt{\frac{1}{N^{regret}_k(s, a)}},
\]
where $L_k = \ln(5SAK_{m_k}H_1/p_1)$, $N^{regret}_k(s, a)$ is the total number of visits to $(s,a)$  in the learning phase before episode $k$, and $K_{m^k}$ is the total number of episodes in which the model $m^k$ is given to the learner. However, $K_{m^k}$ is unknown to the learner. We instead upper bound $K_{m^k}$ by $K$ and modify the bonus term as
\begin{equation}
b'_k(s, a) = 7H_1L\sqrt{\frac{1}{N^{regret}_k(s, a)}}
\label{eq:bonus}
\end{equation}
where $L = \ln(5SAKHM/p_1)$. Since $b'_k \geq b_k$, this algorithm still retain the optimism principle needed for UCBVI-CH. 
The total regret of each model in $\mathcal{M}$ is bounded by the following result, whose proof is in Appendix~\ref{appendix:perModelRegret}.

\begin{mylemma}
With probability at least $1-p_1$, applying UCBVI-CH with the bonus term $b'_k$ defined in Equation~\ref{eq:bonus}, each task $m$ in $\mathcal{M}$ has a total regret of
\[
\text{Regret}(m, K_m) \leq K_m(H_0 + D) + 67H_1^{3/2}L\sqrt{SAK_m} + 15S^2A^2H_1^2L^2
\]
\label{lemma:perModelUCBVICH}
\end{mylemma}

\TheoremMainRegret*
\begin{proof}
Summing up the regret for all $m \in \mathcal{M}$ and applying the Cauchy-Schwarz inequality, Lemma~\ref{lemma:perModelUCBVICH} together with Lemma~\ref{lemma:noIncorrectClusters} and Lemma~\ref{lemma:computeH0} imply that with probability $1 - p$, the total regret is bounded by 

\begin{equation}
\text{Regret}(K) \leq K(H_0 + D) + 67H_1L\sqrt{MSAKH_1} + 15MS^2AH_1^2L^2.
\label{eq:FirstStepFullRegret}
\end{equation}

Note that the bound in Equation~\ref{eq:FirstStepFullRegret} is tighter than the bound in Theorem~\ref{theorem:mainRegret1}. To obtain the bound in Theorem~\ref{theorem:mainRegret1}, notice that $D \leq \tilde{D} \leq H_0$ and thus $K(H_0 + D) \leq K(H_0 + H_0) = 2KH_0$.
\end{proof}

\TheoremRegretAlgorithmNoGamma*
\begin{proof} In stage 1, as the distinguishing set has size $|\tilde{\Gamma}| = SA$, the number of time steps needed in the clustering phase is
\[
H_{0,1} = 12\tilde{D}|\tilde{\Gamma}|N_{1} = 12DSAN_{1},
\]
where $N_{1} = \frac{256}{\lambda^2}\max\{S, \ln(\frac{3KSA}{p})\}$.

In stage 2, the length of the clustering phase is
\[
H_{0,2} = 12\tilde{D}|\hat{\Gamma}|N_{2},
\]
where $N_{2} = \frac{256}{\lambda^2}\max\{S, \ln(\frac{3K|\hat{\Gamma}|}{p})\}$.


Substituting $H_{0,1}$ and $H_{0,2}$ into Theorem~\ref{theorem:mainRegret1}, we obtain the regret bound of stage 1 and stage 2:
\begin{equation*}
\begin{split}
\text{Regret}_{Stage1} \leq 2K_1H_{0,1} &+ 67(H_{1,1})^{3/2}L_{1}\sqrt{MSAK_1} + 15MS^2A(H_{1,1})^2L_{1}^2,
\end{split}
\end{equation*}
where $L_{1} = \ln(\frac{15MSAKH_{1,1}}{p})$ and $H_{1,1} = H - H_{0,1}$.
\[
\begin{split}
\text{Regret}_{Stage2} \leq 2K_2H_{0,2} &+ 67H_{1,2}^{3/2}L_2\sqrt{MSAK_2} + 15MS^2AH_{1,2}^2L_2^2,
\end{split}
\]
where $L_2 = \ln(\frac{15MSAKH_{1,2}}{p})$ and $H_{1,2} = H - H_{0,2}$.

Since $H_{0,1} \geq H_{0,2}$, we have $H_{1,1} \leq H_{1,2}$. Using the assumption that $K_1SA < K_2$ and the Cauchy-Schwarz inequality for the sum $\sqrt{K_1} + \sqrt{K_2}$, we obtain

\begin{align}
  \text{Regret}(K) &= \text{Regret}_{Stage1} + \text{Regret}_{Stage2} \\
  &\leq 4KH_{0,2} + 67H_{1,2}^{3/2}L_2\sqrt{2MSAK} + 30MS^2AH_{1,2}^2L_2^2.
\end{align}

By having $|\hat{\Gamma}| \leq {M \choose 2} \leq M^2, H_{1,2} \leq H$ and $\max\{L_1, L_2\} \leq L$, we obtain  
\begin{align}
  \text{Regret}(K) \leq 4KH_{0,M} &+ 67H^{3/2}L\sqrt{2MSAK} + 30MS^2AH^2L^2.
\end{align}
where $H_{0, M} = \frac{3072\tilde{D}M^2}{\lambda^2}\max\{S, \ln(\frac{3KM^2}{p})\}$.

\end{proof}

\section{Per-model Regret analysis}
\label{appendix:perModelRegret}
First, we prove the following lemma which upper bound the per-episode regret as a function of $H_0$ and the regret of the clustering phase.

\begin{mylemma}
    \label{lemma:perEpRegret}
    The regret of Algorithm~\ref{algo:aomtrl} in episode $k$ is 
    \[
    \Delta_k = [V_1^{k,*} - V_1^{\pi_k}](s^k_1) \leq H_0 + D + \max_{s \in \mathcal{S}}[V_{H_0+1}^{k,*} - V_{H_0+1}^{\pi_k}](s).
    \]
\end{mylemma}
\begin{proof}
Denote by $\Pr(s^k_h = s \mid s_1, \pi)$ the probability of visiting state $s$ at time $h$ when the learner follows a (possibly non-stationary) policy $\pi$ in model $m^k$ starting from state $s_1$. The regret of task $m$ in a single episode $k \in \mathcal{K}_m$ can be written as
\[
    \begin{split}
        \Delta_k &= [V^{k,*}_1 - V^{\pi_k}_1](s^k_1) \\
        &= E[\sum_{h=1}^H r(s_h, a_h) \mid s_1 = s^k_1, a_h = \pi_k^*(s_h)] - E[\sum_{h=1}^H r(s_h, a_h) \mid s_1 = s^k_1, a_h = \pi_k(s_h)] \\
        &= \left(E[\sum_{h=1}^{H_0} r(s_h, a_h) \mid s_1 = s^k_1, a_h = \pi_k^*(s_h)] + \sum_{s \in \mathcal{S}}{\Pr}_m(s^k_{H_0+1}=s\mid s^k_1, \pi^*_k)V_{H_0+1}^{k,*}(s)\right) \\
        &\quad - \left(E[\sum_{h=1}^{H_0} r(s_h, a_h) \mid s_1 = s^k_1, a_h = \pi_k(s_h)] + \sum_{s \in \mathcal{S}}{\Pr}_m(s^k_{H_0+1}=s\mid s^k_1, \pi_k)V_{H_0+1}^{\pi_k}(s)\right) \\
        &\leq H_0 + \sum_{s \in \mathcal{S}}{\Pr}_m(s^k_{H_0+1}=s\mid s^k_1, \pi^*_k)V_{H_0+1}^{k,*}(s) - \sum_{s \in \mathcal{S}}{\Pr}_m(s^k_{H_0+1}=s\mid s^k_1, \pi_k)V_{H_0+1}^{\pi_k}(s) \\
        &= H_0 + \left(\sum_{s \in \mathcal{S}}{\Pr}_m(s^k_{H_0+1}=s\mid s^k_1, \pi^*_k)V_{H_0+1}^{k,*}(s) - \sum_{s \in \mathcal{S}}{\Pr}_m(s^k_{H_0+1}=s\mid s^k_1, \pi_k)V_{H_0+1}^{k,*}(s)\right) \\
        &\quad + \sum_{s \in \mathcal{S}}{\Pr}_m(s^k_{H_0+1}=s\mid s^k_1, \pi_k)[V_{H_0+1}^{k,*} - V^{\pi_k}_{H_0+1}](s) \\
        &\leq H_0 + \underbrace{\left( \max_{s \in \mathcal{S}}V_{H_0+1}^{k,*}(s) - \min_{s \in \mathcal{S}}V_{H_0+1}^{k,*}(s) \right)}_{(\clubsuit)} + \max_{s \in \mathcal{S}}[V_{H_0+1}^{k,*} - V^{\pi_k}_{H_0+1}](s).
    \end{split}
\]
The first inequality follows from the assumption that $r(s, a) \in [0,1]$ for all $(s,a)$. The second inequality follows the fact that 
\[
	\begin{split}
	\sum_{s \in \mathcal{S}}{\Pr}_m(s^k_{H_0+1}=s\mid s^k_1, \pi^*_k)V_{H_0+1}^{k,*}(s) &\leq \sum_{s \in \mathcal{S}}{\Pr}_m(s^k_{H_0+1}=s\mid s^k_1, \pi^*_k)\max_{x \in \mathcal{S}}V_{H_0+1}^{k,*}(x) \\
	&= \left(\max_{x \in \mathcal{S}}V_{H_0+1}^{k,*}(x)\right)\sum_{s \in \mathcal{S}}{\Pr}_m(s^k_{H_0+1}=s\mid s^k_1, \pi^*_k) \\
	&= \max_{x \in \mathcal{S}}V_{H_0+1}^{k,*}(x),
	\end{split}
\]

and
\[
	\begin{split}
		\sum_{s \in \mathcal{S}}{\Pr}_m(s^k_{H_0+1}=s\mid s^k_1, \pi_k)V_{H_0+1}^{k,*}(s) &\geq \sum_{s \in \mathcal{S}}{\Pr}_m(s^k_{H_0+1}=s\mid s^k_1, \pi_k)\min_{x \in \mathcal{S}}V_{H_0+1}^{k,*}(x) \\
		&= \left(\min_{x \in \mathcal{S}}V_{H_0+1}^{k,*}(x)\right)\sum_{s \in \mathcal{S}}{\Pr}_m(s^k_{H_0+1}=s\mid s^k_1, \pi_k) \\
		&= \min_{x \in \mathcal{S}}V_{H_0+1}^{k,*}(x).
		\end{split}
\]

Furthermore, since $V^{k,*}_{H_0+1}(s) \geq V_{H_0+1}^{\pi_k}(s)$ for all $s \in \mathcal{S}$, we have
		\[
		\begin{split}
			\sum_{s \in \mathcal{S}}{\Pr}_m(s^k_{H_0+1}=s\mid s^k_1, \pi_k)[V_{H_0+1}^{k,*} - V^{\pi_k}_{H_0+1}](s) &\leq \sum_{s \in \mathcal{S}}{\Pr}_m(s^k_{H_0+1}=s\mid s^k_1, \pi_k)\max_{x \in \mathcal{S}}[V_{H_0+1}^{k,*} - V^{\pi_k}_{H_0+1}](x) \\
			&= \max_{x \in \mathcal{S}}[V_{H_0+1}^{k,*} - V^{\pi_k}_{H_0+1}](x)\sum_{s \in \mathcal{S}}{\Pr}_m(s^k_{H_0+1}=s\mid s^k_1, \pi_k) \\
			&= \max_{x \in \mathcal{S}}[V_{H_0+1}^{k,*} - V^{\pi_k}_{H_0+1}](x).
		\end{split}	
		\]
For each state $s$, the value of $V_h^{k,*}(s)$ is the expected total $(H-h)$-step reward of an optimal non-stationary $(H-h)$ step policy starting in state $s$ on the MDP $m$. Thus, the term $(\clubsuit)$ represents the \textit{bounded span} of the finite-step value function in MDP $m$. Applying equation 11 of~\cite{UCRL2},  the span of the value function is bounded by the diameter of the MDP. We obtain for all $h$
\[
    \max_{s \in \mathcal{S}}V_h^{k,*}(s) - \min_{s \in \mathcal{S}}V^{k,*}_h(s) \leq D.
\]

It follows that
\[
\Delta_k \leq H_0 + D + \max_{s \in \mathcal{S}}[V^{k,*}_{H_0+1} - V^{\pi_k}_{H_0+1}](s).
\]

\end{proof}
Denote $\mathcal{K}_m$ the set of episodes where the model $m$ is given to the learner. The total regret of the learner in episodes $\mathcal{K}_m$ is 
\[
    \begin{split}
      \mathrm{Regret}(m, K_m) &= \sum_{k \in \mathcal{K}_m} \Delta_k    \\
                       &\leq K_m(H_0 + D) + \underbrace{\sum_{k \in \mathcal{K}_m}\max_{s \in S}[V^{k,*}_{H_0+1} - V^{\pi_k}_{H_0+1}](s)}_{(\heartsuit)}.
    \end{split}
\]

The policy $\pi_k$ from time step $H_0 +1$ to $H$ is the UCBVI-CH algorithm~\citep{Azar2017}. Therefore, the term $(\heartsuit)$ corresponds to the total regret of UCBVI-CH in an adversarial setting in which the starting state $s^k_1$ in each episode is chosen by an adversary that maximizes the regret in each episode. In Appendix~\ref{appendix:UCBVI-CH}, we given a simplified analysis for UCBVI-CH and show that with probability at least $1 - p_1/M$,
\begin{equation}
    (\heartsuit) = \sum_{k \in \mathcal{K}_m}\max_{s \in S}[V^{k,*}_{H_0+1} - V^{\pi_k}_{H_0+1}](s) \leq 67H_1^{3/2}L\sqrt{SAK_m} + 15S^2A^2H_1^2L^2 .
    \label{eq:sumUCBVI-CH}
\end{equation}

The proof of Lemma~\ref{lemma:perModelUCBVICH} is completed by plugging the bound of $(2)$ in Equation~\ref{eq:sumUCBVI-CH} to obtain 
\[
    \begin{split}
      \mathrm{Regret}(m, K_m) &= \sum_{k \in \mathcal{K}_m} \Delta_k    \\
                       &\leq K_m(H_0 + D) + 67H_1^{3/2}L\sqrt{SAK_m} + 15S^2A^2H_1^2L^2  .
    \end{split}
\]

\section{A simplified analysis for UCBVI-CH}
\label{appendix:UCBVI-CH}

\begin{algorithm2e}[ht]
  \SetAlgoNoEnd
    \KwIn{Failure probability $p$}
    Initialize an empty collection $\mathcal{B}$\;
    \For{episode $k = 1, \dots, K$:}{
        \SetKwFunction{ValueIteration}{UCB-Q-Values}
        $Q_{k,h} = $ \ValueIteration($\mathcal{B}, p$)\;
        \For{$h = 1, \dots, H$: }{  
        Take action $a_{k,h} = \argmax_a Q_{k,h}(s^k_{h}, a)$\;
        Add $(s^k_{h}, a^k_{h}, s^k_{h+1})$ to $\mathcal{B}$ \;
        }
    }
    \caption{UCBVI}
    \label{algo:UCBVI}
    \end{algorithm2e}
    
    \begin{algorithm2e}[ht]
      \SetAlgoNoEnd
    \KwIn{Collection $\mathcal{B}$, probability $p$}
    Compute, for all $(s, a, s') \in \mathcal{S \times A \times S}$\newline
    $N_k(s,a,s') = \sum_{(x,a',y) \in \mathcal{B}}\mathbb{I}(x=s,a'=a,y=s')$\newline
    $N_k(s,a) = \sum_{s' \in S}N_k(s,a,s')$\;
    For all $(s,a) \in \{(s,a): N_k(s,a) > 0\}$, compute \newline
    $\hat{P}_k(s' \mid s,a) = \frac{N_k(s,a,s')}{N_k(s,a)}$\newline
    $b_{k,h}(s,a) = 7HL\sqrt{\frac{1}{N_k(s,a)}}$ where $L = \ln(5SAKH/p)$\;
    Initialize $V_{k, H+1}(s) = 0$ for all $x \in \mathcal{S}$\;
    \For{$h = H, H-1, \dots, 1$:}{
        \For{$(s, a) \in \mathcal{S \times A}$}{
            \uIf{$N_k(s, a) > 0$}
            {
                $Q_{k,h}(s,a) = \min\{H, r(s,a) + \left(\sum_{s' \in \mathcal{S}}\hat{P}_k(s'\mid s,a)V_{k,h+1}(s')\right) + b_{k,h}(s,a)\}$
            }
            \uElse{
                $Q_{k,h} = H$
            }
            $V_{k,h}(s) = \max_{a}Q_{k,h}(s,a)$
        }
    }
    \caption{UCB-Q-Values with Hoeffding bonus}
    \label{algo:ValueIteration}
    \end{algorithm2e}

In section, we construct a simplified analysis for the UCBVI-CH algorithm in~\cite{Azar2017}. The proof largely follows the existing constructions in~\cite{Azar2017}, with two differences: the definition of ``typical'' episodes and the analysis are tailored specifically for the Chernoff-type bonus of UCBVI-CH, without being complicated by handling of the variances for the Bernstein-type bonus of UCBVI-BF in~\cite{Azar2017}. For completeness, the full UCBVI-CH algorithm from~\cite{Azar2017} is shown in Algorithms~\ref{algo:UCBVI} and~\ref{algo:ValueIteration}.

\textbf{Notation}. In this section, we consider the standard single-task episodic RL setting in ~\cite{Azar2017} where the learner is given the same MDP $(\mathcal{S, A}, H, P, r)$ in $K$ episodes. We assume the reward function $r: \mathcal{S \times A} \mapsto [0, 1]$ is deterministic and known. The state and action spaces $\mathcal{S}$ and $\mathcal{A}$ are discrete spaces with size $S$ and $A$, respectively. Denote by $p$ the failure probability and let $L = \ln(5SAKH/p)$. We assume the product $SAKH$ is sufficiently large that $L > 1$.

 Let $V_1^*$ denote the optimal value function and $V^{\pi_k}_1$ the value function of the policy $\pi_k$ of the UCBVI-CH agent in episode $k$. The regret is defined as follows.
\begin{equation}
  \mathrm{Regret}(K) = \sum_{k=1}^K \delta_{k,1},
\label{eq:singleTaskRegret}
\end{equation}
where $\delta_{k,h} =  [V_h^* - V^{\pi_k}_h](s^k_h)$.

Denote by $N_k(s,a)$ the number of visits to the state-action pair $(s,a)$ up to the beginning of episode $k$. 

We call an episode $k$ ``typical'' if all state-action pairs visited in episode $k$ have been visited at least $H$ times at the beginning of episode $k$. The set of typical episodes is defined as follows.
\begin{equation}
[K]_{typ} = \{i \in [K]: \forall h \in [H], N_i(s^i_h, a^i_h) \geq H\}.    
\label{eq:TypDef}
\end{equation}
Equation~\ref{eq:singleTaskRegret} can be written as
\begin{equation}
\begin{split}
  \mathrm{Regret}(K) &=  \sum_{k \notin [K]_{typ}}\delta_{k,1} + \sum_{k \in [K]_{typ}}\delta_{k,1} \\
    &\leq \sum_{k \notin [K]_{typ}}H + \sum_{k \in [K]_{typ}}\delta_{k,1} \\
    &\leq SAH^2 + \sum_{k \in [K]_{typ}}\delta_{k,1}.
\end{split}  
\label{eq:sumNonTypAndTyp}  
\end{equation}

The first inequality follows from the trivial upper bound of the regret in an episode $\delta_{k,1} \leq H$. The second inequality comes from the fact that each state-action pair can cause at most $H$ episodes to be non-typical; therefore there are at most $SAH$ non-typical episodes.

Next, we have:
\begin{equation}
\begin{split}
    \sum_{k \in [K]_{typ}}\delta_{k,1} &= \sum_{k}^K \delta_{k,1} \mathbb{I}\{k \in [K]_{typ}\}.
    \label{eq:sumTyp}
\end{split}    
\end{equation}

From here we write $\mathbb{I}_k = \mathbb{I}\{k \in [K]_{typ}\}$ for brevity.

Lemma 3 in~\cite{Azar2017} implies that, for all $k \in [K]$,
\begin{equation}
\delta_{k,1} \leq e\sum_{h=1}^{H}\left[\varepsilon_{k, h} + 2\sqrt{L}\bar{\varepsilon}_{k, h} + c_{1, k, h} + b_{k, h} + c_{4, k, h}\right].
\label{eq:Lemma3}
\end{equation}
where $c_{4,k,h} = \frac{4SH^2L}{N_k(s^k_h, a^k_h)}$, $\varepsilon_{k,h}$ and $\bar{\varepsilon}_{k,h}$ are martingale difference sequences which, by Lemma 5 in~\cite{Azar2017}, satisfy
\begin{equation}
    \begin{split}
\sum_{k=1}^K\sum_{h=1}^{H}\varepsilon_{k,h} &\leq H\sqrt{KHL} \\
\sum_{k=1}^K\sum_{h=1}^{H}\bar{\varepsilon}_{k,h} &\leq \sqrt{KH},
    \end{split}
    \label{eq:sumMDS}
\end{equation}
and $c_{1,k,h}$ is a confidence interval to be defined later.

Plugging Equation~\ref{eq:Lemma3} into Equation~\ref{eq:sumTyp} and combining with Equation~\ref{eq:sumMDS}, we obtain:
\[
    \begin{split}
        \sum_{k \in [K]_{typ}}\delta_{k,1} &\leq e\sum_{k=1}^K \left( \sum_{h=1}^{H}\left[\varepsilon_{k, h} + 2\sqrt{L}\bar{\varepsilon}_{k, h} + c_{1, k, h} + b_{k, h} + c_{4, k, h}\right] \right) \mathbb{I}_k \\
        &= e\left[ \left( \sum_{k=1}^K \mathbb{I}_k \sum_{h=1}^H (\varepsilon_{k,h} + 2\sqrt{L}\bar{\varepsilon}_{k,h}) \right) + 
            \left(  
                \sum_{k=1}^K \mathbb{I}_k \sum_{h=1}^H (b_{k,h} + c_{1,k,h} + c_{4,k,h})
            \right)
        \right] \\
        &\leq e\left[ 
            \left( \sum_{k=1}^K\sum_{h=1}^H (\varepsilon_{k,h} + 2\sqrt{L}\bar{\varepsilon}_{k,h}) \right) + 
            \left(
                \sum_{k=1}^K \sum_{h=1}^H (b_{k,h}\mathbb{I}_k + c_{1,k,h}\mathbb{I}_k + c_{4,k,h}\mathbb{I}_k)
            \right)
        \right] \\
        &\leq e\left[
            \left( H\sqrt{KHL} + 2\sqrt{L}\sqrt{KH}\right) + 
            \left(
                \sum_{k=1}^K \sum_{h=1}^H (b_{k,h}\mathbb{I}_k + c_{1,k,h}\mathbb{I}_k + c_{4,k,h}\mathbb{I}_k)
            \right)
        \right] \\
        &= e\left[
            \left( (H+2)\sqrt{KHL}\right) + 
            \left(
                \sum_{k=1}^K \sum_{h=1}^H (b_{k,h}\mathbb{I}_k + c_{1,k,h}\mathbb{I}_k + c_{4,k,h}\mathbb{I}_k)
            \right)
        \right]
    \end{split}
\]

Note that the second inequality follows from the fact that $\mathbb{I}_k \leq 1$, and the last inequality follows directly from Equation~\ref{eq:sumMDS}.

Let $\mathbb{I}_{k,h} = \mathbb{I}\{N_k(s^k_h, a^k_h) \geq H\}$. By the definition of a ``typical'' episode, $\mathbb{I}_k = 1$ implies that $\mathbb{I}_{k,h}=1$ for all $h$. It follows that $\mathbb{I}_k \leq \mathbb{I}_{k,h}$. Thus, 
\begin{equation}
    \sum_{k \in [K]_{typ}}\delta_{k,1} \leq e\left((H+2)\sqrt{KHL} + \sum_{i=1}^{K}\sum_{j=1}^{H}(b'_{k, h} + c'_{1,k,h} + c'_{4,k, h}) \right),
    \label{eq:sumbprimecprime}
\end{equation}
where $b'_{k,h} = b_{k,h}\mathbb{I}_{k,h}$, $c'_{1,k,h} = c_{1,k,h}\mathbb{I}_{k,h}$ and $c'_{4,k,h} = c_{4,k,h}\mathbb{I}_{k,h}$.

Next, we compute $c_{1, k,h}$. In Equation (32) in~\cite{Azar2017}, $c_{1, k, h}$ corresponds to the confidence interval of 
\[    
(\hat{P}_h^\pi - P_h^\pi)V^*_{h+1}(s^k_h) = \sum_{s' \in \mathcal{S}}\left[
    \hat{P}(s' \mid s^k_h, a^k_h) - P_h(s' \mid s^k_h, a^k_h)
\right]V^*_{h+1}(s').
\]
Equation (9) in~\cite{Azar2017} computes a confidence interval for this term using the Bernstein inequality. Instead, we use the Hoeffding inequality and obtain
\begin{equation}
    [(\hat{P}_h^\pi - P_h^\pi)V^*_{h+1}] \leq H\sqrt{\frac{L}{2N_k(s^k_h, a^k_h)}} = c_{1,k,h}.
    \label{eq:c1kh}
\end{equation}

Combining Equations~\ref{eq:c1kh}, ~\ref{eq:sumbprimecprime} and~\ref{eq:sumNonTypAndTyp}, the total regret is bounded as
\begin{equation}
\begin{split}
\text{Regret} \leq SAH^2 + e\left((H+2)\sqrt{KHL} + \underbrace{\sum_{k=1}^{K}\sum_{h=1}^{H}(b'_{k,h} + c'_{1,k,h} + c'_{4,k,h})}_{(a)} \right)
\end{split}
\label{eq:AppendixRegret}
\end{equation}
where $b'_{k,h} = \frac{7HL\mathbb{I}_{k,h}}{\sqrt{N_k(s^k_h, a^k_h)}}, c'_{1,k,h} = \frac{H\sqrt{L}\mathbb{I}_{k,h}}{\sqrt{2N_k(s^k_h, a^k_h)}}$ and $c'_{4, k,h} = \frac{4SH^2L\mathbb{I}_{k,h}}{N_k(s^k_h, a^k_h)}$.

We focus on the third and dominant term $(a)$. As $b_{k,h} \geq c_{1,k,h}$, this term can be upper bounded by
\begin{equation}
\begin{split}
    (a) &\leq  \sum_{k=1}^{K}\sum_{h=1}^{H} \left[ \frac{8HL\mathbb{I}_{k,h}}{\sqrt{N_k(s^k_h, a^k_h)}} + \frac{4SH^2L\mathbb{I}_{k,h}}{N_k(s^k_h, a^k_h)} \right] \quad \text{(since } L > 1 \text{)} \\
    &= 8HL\underbrace{\sum_{i=1}^{K}\sum_{j=1}^{H}\frac{\mathbb{I}_{k,h}}{\sqrt{N_k(s^k_h, a^k_h)}}}_{(b)} + 4SH^2L\underbrace{\sum_{i=1}^{K}\sum_{j=1}^{H}\frac{\mathbb{I}_{k,h}}{N_k(s^k_h,a^k_h)}}_{(c)}.
\end{split}
\label{eq:AppendixAa}
\end{equation}

We bound $(b)$ and $(c)$ separately.

First, we bound $(b)$. We introduce the following lemma, which is an analogy to Lemma 19 in~\cite{UCRL2} in the finite-horizon setting. 
\begin{mylemma} Let $H \geq 1$.
For any sequence of numbers $z_1, \dots, z_n$ with $0 \leq z_k \leq H$, consider the sequence $Z_0, Z_1, \dots Z_n$ defined as
\[
\begin{split}
Z_0 &\geq H \\
Z_k &= Z_{k-1} + z_k \qquad \text{for } k \geq 1.
\end{split}
\]

Then, for all $n \geq 1$,
\[
\sum_{k=1}^n \frac{z_k}{\sqrt{Z_{k-1}}} \leq (\sqrt{2}+1)\sqrt{Z_n}.
\]
\label{lemma:My19a}
\end{mylemma}

Using Lemma~\ref{lemma:My19a}, we can bound $(b)$ by Lemma~\ref{lemma:AppendixAb}.
\begin{mylemma}
Denote $v_i(s,a) = \sum_{j = 1}^H \mathbb{I}(a_{i,j} = a, s_{i,j} = s)$ the number of times the state-action pair $(s,a)$ is visited during episode $i$, and let $\tau(s,a) = \argmin_{k \in [K]}\{N_k(s,a) \geq H\}$ be the first episode where the state-action pair $(s,a)$ is visited at least $H$ times. Then,
\begin{equation}
\begin{split}
(b) \leq (\sqrt{2}+1) \sqrt{SAKH}.
\end{split}
\end{equation}
\label{lemma:AppendixAb}
\end{mylemma}
\begin{proof}
By definition, $N_i(s,a) = \sum_{k=1}^{i-1}v_k(s,a)$. Regrouping the sum in $(b)$ by $(s,a)$, we have
\[
\begin{split}
(b) &= \sum_{s,a}\sum_{i=1}^{K} \frac{v_i(s,a)}{\sqrt{N_i(s,a)}}\mathbb{I}\{N_i(s,a) \geq H\} \\
&= \sum_{s,a} \left(\sum_{i=1}^{\tau(s,a)-1}\frac{v_i(s,a)}{\sqrt{N_i(s,a)}}\mathbb{I}\{N_i(s,a) \geq H\} + \sum_{i=\tau(s,a)}^{K} \frac{v_i(s,a)}{\sqrt{N_i(s,a)}} \right) \\
&= \sum_{s,a}\sum_{i=\tau(s,a)}^{K} \frac{v_i(s,a)}{\sqrt{N_i(s,a)}} \\
&\leq \sum_{s,a}(\sqrt{2} + 1)\sqrt{N_{K}(s,a) + v_K(s,a)} \\
&\leq (\sqrt{2}+1) \sqrt{SAKH}.
\end{split}
\]
where the last two inequalities follow from Lemma~\ref{lemma:My19a}, the Cauchy-Schwarz inequality and the fact that $\sum_{s,a}{N_{K}}(s,a) \leq KH$.
\end{proof}

In order to bound the term $(c)$ in Equation~\ref{eq:AppendixAa}, we use the following lemma, which is a variant of Lemma~\ref{lemma:My19a} and was stated in~\cite{Azar2017} without proof.

\begin{mylemma} Let $H \geq 1$.
For any sequence of numbers $z_1, \dots, z_n$ with $0 \leq z_k \leq H$, consider the sequence $Z_0, Z_1, \dots Z_n$ defined as
\[
\begin{split}
Z_0 &\geq H \\
Z_k &= Z_{k-1} + z_k \qquad \text{for } k \geq 1.
\end{split}
\]

Then, for all $n \geq 1$,
\[
\sum_{k=1}^n \frac{z_k}{Z_{k-1}} \leq \sum_{j=1}^{Z_n-Z_0}\frac{1}{j}\leq \ln(Z_n-Z_0) + 1.
\]
\label{lemma:harmonicSum}
\end{mylemma}
\begin{proof}
The second half follows immediately from existing results for the partial sum of the harmonic series. We prove the first half of the inequality by induction.  By definition of the two sequences, $Z_k \geq H \geq 1$ and $z_k \leq H \leq Z_{k-1}$ for all $k$. At $n = 1$, if $z_1 = 0$ then the inequality trivially holds. If $z_1 > 0$, then $Z_1 - Z_0 = z_1$ and
\[
\frac{z_1}{Z_0} \leq \frac{z_1}{H} = \left( \underbrace{\frac{1}{H} + \dots + \frac{1}{H}}_{z_1 \text{ terms}} \right) \leq 1 + \frac{1}{2} + \dots + \frac{1}{z_1}
\]
since $z_1 \leq H$.

For $n > 1$, by the induction hypothesis, we have
\[
\begin{split}
\sum_{k=1}^n \frac{z_k}{Z_{k-1}} &= \sum_{k=1}^{n-1} \frac{z_k}{Z_{k-1}} + \frac{z_n}{Z_{n-1}} \\
&\leq \left( \sum_{j=1}^{Z_{n-1}-Z_0}\frac{1}{j} \right) +  \frac{z_n}{Z_{n-1}}\\
&= \left( \sum_{j=1}^{Z_{n-1}-Z_0}\frac{1}{j} \right) + \left( \underbrace{ \frac{1}{Z_{n-1}} + \dots + \frac{1}{Z_{n-1}} }_{z_n \text{terms}} \right) \\
&\leq \left( \sum_{j=1}^{Z_{n-1}-Z_0}\frac{1}{j} \right) + \left( \frac{1}{Z_{n-1} - Z_0 + 1} + \dots + \frac{1}{Z_{n-1} - Z_0 + z_n} \right) \\
&= \sum_{j=1}^{Z_{n} - Z_0} \frac{1}{j},
\end{split}
\]
where the last inequality follows from $z_n \leq Z_0$.
Therefore, the induction hypothesis holds for all $n \geq 1$.
\end{proof}

Using Lemma~\ref{lemma:harmonicSum}, the term $(c)$ can be bounded similarly to term $(b)$ as follows:

\begin{mylemma}
With $v_i(s,a)$ and $\tau(s,a)$ defined in Lemma~\ref{lemma:AppendixAb}, we have
\[
(c) \leq SAL + SA.
\]
\label{lemma:AppendixAc}
\end{mylemma}
\begin{proof}
We write $(c)$ as
\[
\begin{split}
(c) &= \sum_{i=1}^{K}\sum_{j=1}^{H} \frac{\mathbb{I}\{N_i(s,a) \geq H\}}{N_i(s_{i,j}, a_{i, j})} \\
&= \sum_{s,a}\sum_{i=1}^{K}\frac{v_i(s,a)}{N_i(s,a)}\mathbb{I}\{N_i(s,a) \geq H\} \\
&\leq \sum_{s,a}\left(\sum_{i=1}^{\tau(s,a)-1}\frac{v_i(s,a)}{N_i(s,a)}\mathbb{I}\{N_i(s,a) \geq H\} + \sum_{i=\tau(s,a)}^{K} \frac{v_i(s,a)}{N_i(s,a)} \right) \\
&= \sum_{s,a}\sum_{i=\tau(s,a)}^{K} \frac{v_i(s,a)}{N_i(s,a)} \\
&\leq \sum_{s,a} \left( \ln\left(N_K(s,a) + v_K(s,a) - N_{\tau(s,a)}(s,a)\right) + 1 \right)
\end{split}
\]
where the last inequality follows from Lemma~\ref{lemma:harmonicSum}. Trivially bounding the logarithm term by $\ln(KH)$, we obtain
\begin{equation*}
(c) \leq SA\ln(KH) + SA \leq SAL + SA.
\end{equation*}
\end{proof}

Combining Lemma~\ref{lemma:AppendixAb} and Lemma~\ref{lemma:AppendixAc}, we obtain
\[
\begin{split}
(a) &\leq 8HL((\sqrt{2}+1)\sqrt{SAKH}) + 4SH^2L(SAL + SA) \\
&\leq 20HL\sqrt{SAKH} + 5S^2AH^2L^2.
\end{split}
\]
Substituting this into Equation~\ref{eq:AppendixRegret}, we obtain
\[
\begin{split}
\text{Regret} &\leq SAH^2 + e(H+2)\sqrt{KHL} + e20HL\sqrt{SAKH} + e5S^2AH^2L^2 \\
&\leq 67HL\sqrt{SAKH} + 15S^2AH^2L^2.
\end{split}
\]

\section{Removing the assumption on the hitting time}
\label{appendix:removingHittingTime}

  GOSPRL~\citep[Lemma 3]{Tarbouriech2021a} guaranteed that in the undiscounted infinite horizon setting, with $H_0 = O(\frac{DS^2A}{\lambda^2})$, Lemma~\ref{lemma:computeH0} holds with high probability. Thus, in the episodic finite horizon setting, by setting $H_0 = c\frac{DS^2A}{\lambda^2}$ for some appropriately large constant $c > 0$ and applying GOSPRL in each episode we obtain a tight bound in the dependency of $K$ and $\lambda$ for communicating MDPs. One difficulty in this approach is both $c$ and $D$ are unknown. One possible way to overcome this is to apply the doubling-trick as following: at the beginning of episode $k$, we set $H_0 = c_k\frac{S^2A}{\lambda^2}$, where $c_1 = 1$. If the learner successfully  visits every state-action pair at least $N$ times after $H_0$ steps, we set $c_{k+1} = c_k$. Otherwise, $c_{k+1} = 2c_k$. There are at most $\log_2{(cD)}$ episodes with failed exploration until $c_k$ is large enough so that with high probability, all the subsequent episodes will have successful explorations. Moreover, the horizons of the clustering and learning phases change at most $\log_2(cD)$ times. The full analysis of this approach is not in the scope of this paper and is left to future work.  

  \section{Using samples in both phases for regret minimization}
  \label{sec:externalsamples}
  
  One of the results from previous works on the stochastic infinite-horizon multi-task setting~\citep{BrunskillAndLi2013} is that in the cluster-then-learn paradigm, the samples collected in the their first stage (before all models have been seen at least once) can be used to accelerate the learning in their second stage (after all models have been seen at least once). In this work, we study the similar effects at the phase level. Specifically, in the finite horizon setting, the clustering phase is always followed by the learning phase; therefore it is desirable to use the samples collected in the clustering phase to improve the regret bound of the learning phase.
  
  Our goal is to improve the regret of stage 1 in Algorithm~\ref{algo:unknownGamma}. The reason that we focus on Stage 1 is two-fold:
  
  \begin{itemize}
  \item In case Assumption~\ref{assumption:K1} does not hold, i.e. $K_1$ is close to $K$, the total regret is dominated by the regret of stage 1. Given that the length of the clustering phase $H_0$ is already of the same order $O(S^2A)$ with respect to the state-of-the-art bound of the recently proposed GOSPRL algorithm~\citep{Tarbouriech2021a}, without further assumptions we conjecture that it is difficult to improve $H_0$ substantially, and thus we focus on improving the learning phase.
  \item In stage 1, every state-action pair is uniformly visited at least $N$ times before the learning phase. This uniformity allows us to study their impact in a systematic way without any further assumptions.
  \end{itemize}
  
  Using samples collected in both phases for the learning phase in Algorithm~\ref{algo:aomtrl} is equivalent to using the policy
  
  \[
  \pi_k = \texttt{UCBVI-CH}(\mathcal{C}_{id})
  \]
  for the learning phase, since $\mathcal{C}_{id}$ contains both $\mathcal{C}_{id}^{model}$ and $\mathcal{C}_{id}^{regret}$.
  
  \begin{algorithm2e}[t]
  \KwIn{Number of episode $K$, horizon $H$, failure probability $p$, number of external samples for each state-action pair $N$}
  Initialize two empty collections $\mathcal{H}$ and $\mathcal{B}$\;
  \For{episode $k = 1, 2, \dots, K$}
  {
      \For{$(s,a) \in \mathcal{S \times A}$)}
      {
          \For{$counter = 1, 2, \dots, N$}
          {
              The oracle draws $s'$ from $P(\cdot\mid s, a)$\;
              Add $(s, a, s')$ to $\mathcal{B}$\;            
          }
      }
      $\pi_k = \texttt{UCBVI-CH}(\mathcal{H \cup B})$\;
      Observe the starting state $s_1$\;
      \For{$h = 1, 2, \dots, H$}
      {
          Learner takes action $a_h = \pi_k(s_h)$\;
          Observe state $s_{h+1}$\;
          Add $(s_{h}, a_{h}, s_{h+1})$ to $\mathcal{H}$\;        
      }
  }
  \caption{UCBVI-CH with external samples}
  \label{algo:UCBVICHwithExternalSamples}
  \end{algorithm2e}
  
  The regret minimization process in the learning phase is now equivalent to learning single-task episodic RL where at the beginning of each episode, the learner is given $S AN$ more $(s, a, s')$ samples, in which the transition function $P(\cdot\mid s,a)$ of each $(s,a)$ is sampled i.i.d. $N$ times. 
  We extend the UCBVI-CH algorithm in~\cite{Azar2017} to this new setting and obtain Algorithm~\ref{algo:UCBVICHwithExternalSamples}. The bonus function of episode $k$ in UCBVI-CH is set to
  \begin{equation}
  b_{k}(s,a) = 7HL_N\sqrt{\frac{1}{N_{k}(s,a)+kN}},
  \end{equation}
  where $L_N = \ln(5SAK(H+N)/p)$.
  
  The regret of this algorithm is bounded in the following theorem (proved in Appendix~\ref{appendix:UCBVICH-extra}).
  
  \begin{theorem}
  Given a constant $p \in (0, 1)$. With probability at least $1-p$, the regret of Algorithm~\ref{algo:UCBVICHwithExternalSamples} is bounded by
  \[
  \text{Regret}(K) \leq \frac{SAH^2}{N+1} + e(H+1)\sqrt{KL_N} + 60\sqrt{\frac{2H-1}{N+2H-1}} H^{3/2}L_N\sqrt{SAK} + 15\frac{2H-1}{N+2H-1}S^2AH^2L_N^2.
  \]
  \label{theorem:regretExternalSamples}
  \end{theorem}
  
  It can be observed that when $N = 0$, this bound recovers the bound of UCBVI-CH (up to a constant factor). Intuitively, when $N$ is small compared to $H$, then the regret should still be of order $O(H\sqrt{SAKH})$ since most of the useful information for learning still comes from exploring the environment. As $N$ increases, since the logarithmic term $L_N$ increases much slower compared to $O(1/\sqrt{N})$, the dominant term $O(\sqrt{\frac{2H-1}{N+2H-1}} H^{3/2}L_N\sqrt{SAK})$ converges to 0. 
  
  Using Theorem~\ref{theorem:regretExternalSamples} and $H_1 \leq H$, we can directly bound the regret of each model $m$ that is given in $\mathcal{K}_m$:
  
  \begin{mylemma}
  The stage-1 regret of each model $m$ is
  \begin{equation*}
  \begin{split}
  \text{Regret}_{Stage1}(m, K_m) \leq &\frac{SAH_1^2}{N+1} + e(H_1+1)\sqrt{K_mL_{N}} \\
  &+ 60\sqrt{\frac{2H_1-1}{N+2H_1-1}}H_1^{3/2}L_{N}\sqrt{SAK_m} + 15\frac{2H_1-1}{N + H_1 -1}S^2AH_1^2L_{N}^2.
  \end{split}
  \end{equation*}
  where $L_{N} = \ln(5SAK(H+N)/p)$.
  \label{lemma:eachmodelregretExternal}
  \end{mylemma}
  Adding up the bound in Lemma~\ref{lemma:eachmodelregretExternal} for all models $m \in \mathcal{M}$ and applying the Cauchy-Schwarz inequality, we obtain the total regret bound of Stage 1:
  \begin{theorem}
  \[
  \begin{split}
  \text{Regret}_{Stage1} \leq &K_1H_0 + \frac{MSAH_1^2}{N+1} + e(H_1+1)\sqrt{MKL_{N}} \\
  &+ 60\sqrt{\frac{2H_1-1}{N+2H_1-1}}H_1^{3/2}L_{N}\sqrt{MSAK} + 15M\frac{2H_1-1}{N + H_1 -1}S^2AH_1^2L_{N}^2.
  \end{split}
  \]
  \label{theorem:totalStage1External}
  \end{theorem}
  In our setting, recall that $N = O\left(\frac{S}{\lambda^2}\right)$ and $H_0 = O(DSAN) = O(DS^2A/\lambda^2)$. Since we assumed that $SA << H$, we also have $N \ll H_1 = H - H_0$, and thus the bound in Theorem~\ref{theorem:totalStage1External} is an improvement from the bound for stage 1 in the proof of Theorem~\ref{theorem:regretUnknownGamma}, albeit the order stays the same. Intuitively, this means that the length of the learning phase is much larger than the length of the clustering phase, and therefore the learner spends more time on learning the optimal policy. When the length of the learning phase is small compared to $N$, then the samples collected in the clustering phase significantly reduce the regret bound of the learning phase. Therefore, Algorithm~\ref{algo:UCBVICHwithExternalSamples} also accelerates the learning phase after the exploration phase, which is consistent with findings on the stochastic infinite-horizon multi-task setting in~\cite{BrunskillAndLi2013}.

  \section{Proofs for Appendix~\ref{sec:externalsamples}}
  \label{appendix:UCBVICH-extra}
  
  We analyze the regret of the UCBVI-CH algorithm with external samples, where at the beginning of each episode, each state-action pair receives $N \geq 1$ additional samples drawn i.i.d from the transition function $P(\cdot \mid s,a)$.
  
  Adapting from Equation~\ref{eq:AppendixRegret}, the regret of E-UCBVI-CH can be bounded by
  
  \begin{align}
  \begin{split}
  \text{Regret}(K) &\leq \frac{SAH^2}{N+1} + e(H+1)\sqrt{KHL_N} + e\underbrace{\sum_{i=1}^{K}\sum_{j=1}^{H} \left[ \frac{8HL_N\mathbb{I}_{i,j}}{\sqrt{kN + N_i(s_{i,j}, a_{i,j})}} + \frac{4SH^2L_N\mathbb{I}_{i,j}}{kN + N_i(s_{i,j}, a_{i,j})} \right]}_{(a)},
  \end{split}
\end{align}
  where $\mathbb{I}_{i,j} = \mathbb{I}\{N_i(s_{i,j},a_{i,j}) \geq H\}$
  as defined in Appendix~\ref{appendix:UCBVI-CH}.
  
  The first term $\frac{SAH^2}{N+1}$ bounds the total regret of episodes where a state-action pair is visited less than $H$ times: in each episode where a pair $(s,a)$ is visited at least once there are at least $N+1$ more samples of this pair, and therefore there can be at most $\frac{SAH}{N+1}$ such episodes.
  
  Similar to Appendix~\ref{appendix:UCBVI-CH}, we bound $(a)$ by bounding its two components $(b)$ and $(c)$ where
  \[
  (a) = 8HL_N\left( \underbrace{ \sum_{i=1}^{K}\sum_{j=1}^{H} \frac{\mathbb{I}_{i,j}}{\sqrt{kN + N_i(s, a)}}}_{(b)} \right) + 4SH^2L_N\left( \underbrace{\sum_{i=1}^{K}\sum_{j=1}^H \frac{\mathbb{I}_{i,j}}{kN + N_i(s,a)}}_{(c)}\right).
  \]
  
  In order to bound $(b)$, we first prove the following technical lemma, which quantifies the fraction of the regret that is reduced when using external samples.
  \begin{mylemma}
  Suppose two constants $N \geq 1, H \geq 1$ are given.
  For any sequence of numbers $z_1, \dots, z_n$ with $0 \leq z_k \leq H$, consider the sequence $Z_0, Z_1, \dots Z_n$ defined as
  \[
  \begin{split}
  Z_0 &\leq 2H-1 \\
  Z_k &= Z_{k-1} + z_k \qquad \text{for } k \geq 1
  \end{split}
  \]
  Then, for all $k$,
  \[
  \frac{z_k}{\sqrt{kN + Z_{k-1}}} \leq \sqrt{\frac{(k+1)H-1}{kN+(k+1)H-1}} \frac{z_k}{\sqrt{Z_{k-1}}}.
  \]
  \label{lemma:reduceSqrt}
  \end{mylemma}
  \begin{proof}
  If $z_k = 0$, then the claim is trivially true.
  For $z_k > 0$, the claim is equivalent to
  
  \begin{alignat*}{2}
  &&\frac{1}{\sqrt{kN + Z_{k-1}}} &\leq  \sqrt{\frac{(k+1)H-1}{kN+(k+1)H-1}}  \frac{1}{\sqrt{Z_{k-1}}} \\
  \Leftrightarrow && \sqrt{(kN+(k+1)H-1)}\sqrt{Z_{k-1}} &\leq \sqrt{(k+1)H-1}\sqrt{kN + Z_{k-1}} \\
  \Leftrightarrow && Z_{k-1} &\leq (k+1)H-1,
  \end{alignat*}
  which is true, since $Z_{k-1} = Z_0 + \sum_{i=1}^{k-1}z_k \leq Z_0 + \sum_{i=1}^{k-1}H \leq 2H-1 + (k-1)H = (k+1)H - 1$.
  \end{proof}
  
  \begin{mycorollary}
  Suppose two constants $N \geq 1, H \geq 1$ are given.
  For any sequence of numbers $z_1, \dots, z_n$ with $0 \leq z_k \leq H$, consider the sequence $Z_0, Z_1, \dots Z_n$ defined as
  \[
  \begin{split}
  1 \leq Z_0 &\leq 2H-1 \\
  Z_k &= Z_{k-1} + z_k \qquad \text{for } k \geq 1
  \end{split}
  \]
  Then, for all $n \geq 1$,
  \[
  \sum_{k=1}^n \frac{z_k}{\sqrt{kN + Z_{k-1}}} \leq \sum_{k=1}^n \sqrt{\frac{(k+1)H-1}{kN+(k+1)H-1}} \frac{z_k}{\sqrt{Z_{k-1}}} \leq \sqrt{\frac{2H-1}{N+2H-1}} \sum_{k=1}^n \frac{z_k}{\sqrt{kN + Z_{k-1}}}.
  \]
  \label{corollary:reduceSqrt}
  \end{mycorollary}
  \begin{proof}
  The first half of the claim is true, following Lemma~\ref{lemma:reduceSqrt}. We now show that the second half is true. Consider the following function 
  
  \[
  f(x) = \frac{(x+1)H-1}{xN + (x+1)H - 1}
  \]
  
  The derivative is $f'(x) = \frac{N(1-H)}{(xN + (x+1)H - 1)^2}$. Since $H \geq 1$, we have $f'(x) \leq 0~\forall x$, and therefore $f(x)$ is decreasing. It follows that for $k \geq 1$,
  
  \[
  f(k) = \frac{(k+1)H-1}{kN + (k+1)H - 1} \leq f(1) = \sqrt{\frac{2H-1}{N+2H-1}}.
  \]
  \end{proof}
  
  Using Corollary~\ref{corollary:reduceSqrt}, we can bound $(b)$ as following.
  
  \begin{mylemma}
  With $v_i(s,a)$ and $\tau(s,a)$ defined in Lemma~\ref{lemma:AppendixAb}, we have
  
  \[
  (b) \leq \sqrt{\frac{2H-1}{N+2H-1}} (\sqrt{2}+1) \sqrt{SAKH}.
  \]
  \label{lemma:AppendixBb}
  \end{mylemma}
  \begin{proof}
  We can write $(b)$ as follows
  \[
  (b) = \sum_{s,a}\sum_{i=\tau(s,a)}^{K}\frac{v_i(s,a)}{\sqrt{iN + N_i(s,a)}}.
  \]
  
  By definition of $\tau(s,a)$: 
  \[
  N_{\tau(s,a)} = N_{\tau(s,a)-1} + v_{\tau(s,a)-1} \leq H-1 + H = 2H-1.
  \]
  \end{proof}
  
  Applying Corollary~\ref{corollary:reduceSqrt} and Lemma~\ref{lemma:AppendixAb} we obtain
  
  \[
  \begin{split}
  (b) &\leq \sqrt{\frac{2H-1}{N+2H-1}}\sum_{s,a}\sum_{i=\tau(s,a)}^{K}\frac{v_i(s,a)}{\sqrt{N_i(s,a)}} \\
  &\leq \sqrt{\frac{2H-1}{N+2H-1}}(\sqrt{2}+1)\sqrt{SAKH}.
  \end{split}
  \]
  
  Next, we bound $(c)$. Using similar techniques in Lemma~\ref{lemma:reduceSqrt} and Corollary~\ref{corollary:reduceSqrt}, we can show that the following claims are true.
  
  \begin{mylemma}
  Given two constants $N \geq 0, H \geq 1$.
  For any sequence of numbers $z_1, \dots, z_n$ with $0 \leq z_k \leq H$, consider the sequence $Z_0, Z_1, \dots Z_n$ defined as
  \[
  \begin{split}
  1 \leq Z_0 &\leq 2H-1 \\
  Z_k &= Z_{k-1} + z_k \qquad \text{for } k \geq 1
  \end{split}
  \]
  Then, for all $k$,
  \[
  \frac{z_k}{kN + Z_{k-1}} \leq \frac{(k+1)H-1}{kN + (k+1)H-1}\frac{z_k}{Z_{k-1}}.
  \]
  And for all $n \geq 1$,
  \[
  \sum_{k=1}^n \frac{z_k}{kN + Z_{k-1}} \leq \frac{2H-1}{N+2H-1}\sum_{k=1}^n \frac{z_k}{Z_{k-1}}.
  \]
  \label{lemma:reduceNoSqrt}
  \end{mylemma}
  Consequently, $(c)$ is bounded in the following corollary.
  \begin{mycorollary}
  With $v_i(s,a)$ and $\tau(s,a)$ defined in Lemma~\ref{lemma:AppendixAb}, we have
  
  \[
  (c) \leq \frac{2H-1}{N+2H-1}(SAL_N + SA).
  \]
  \label{corollary:reduceNoSqrt}
  \end{mycorollary}
  
  Combining Corollaries~\ref{corollary:reduceSqrt} and ~\ref{corollary:reduceNoSqrt} we obtain
  
  \[
  \begin{split}
  (a) &\leq  8HL\sqrt{\frac{2H-1}{N+2H-1}}((\sqrt{2}+1)\sqrt{SAKH}) + 4SH^2L_N\frac{2H-1}{N+2H-1}(SAL_N + SA) \\
  &\leq \sqrt{\frac{2H-1}{N+2H-1}} 20HL_N\sqrt{SAKH} + \frac{2H-1}{N+2H-1}5S^2AH^2L_N^2.
  \end{split}
  \]
  
  and the total regret is
  
  \[
  \begin{split}
  \text{Regret}(K) &\leq \frac{SAH^2}{N+1} + e(H+1)\sqrt{KL_N}+e\left( \sqrt{\frac{2H-1}{N+2H-1}} 20HL_N\sqrt{SAKH} + \frac{2H-1}{N+2H-1}5S^2AH^2L_N^2       \right).
  \end{split}
  \]
  
  \section{Experimental Details}
  \label{appendix:expsetup}
  
  \textbf{Transition functions}. Figure~\ref{fig:gridworld} illustrate the $4 \times 4$ gridworld environment of the four MDPs in $\mathcal{M}$. The rows are numbered top to bottom from $0$ to $3$. The columns are numbered left to right from $0$ to $3$. The starting state $s_1$ is at position $(1,1)$. In every state, the probability of success of all actions is $0.85$. When an action is unsuccessful, the probability of being in one of the other adjacent cells is equally divided from the remaining probability of $0.15$. There are several exceptions:
  
  \begin{itemize}
      \item In the four corners, if the agent takes an action in the direction of the border then with probability of $0.7$ it will stay in the same corner, and with probability $0.3$ it will end up in the cell in the opposite direction. For example, if the agent is at $(0,0)$ and takes action \texttt{up}, then with probability $0.3$ it will actually goes down to the cell $(1,0)$.
      \item Each of four MDPs have an easy-to-reach corner and three hard-to-reach corners. The easy-to-reach corners in models $m_1, m_2, m_3$ and $m_4$ are $(0,0), (0, 3), (3,0)$ and $(3,3)$, respectively. In each of these model, the probability of success of an action that leads to one of the hard-to-reach corners is $0.2$, except for the $(3,3)$ corner where this probability is $0.3$. For example, in model $m_1$, taking action \texttt{right} in cell $(0,2)$ has probability of success equal to $0.2$ while taking the action $\texttt{down}$ in cell $(2,3)$ has probability of success equal to $0.3$.
      \item On the four edges, any action that takes the agent out of the grid has probability of success equal to 0, and the agent ends up in one of the three adjacent cells with equal probability of $\frac{1}{3}$. For each example, taking action \texttt{up} in position $(0,1)$ will take the agent to one of the three positions $(0, 0), (0, 1)$ and $(1, 1)$ with probability $\frac{1}{3}$.
  \end{itemize}
  
  Under this construction, the seperation level is $\lambda = 1.2999$. One example of a $\lambda$-distinguishing set of optimal size is $\Gamma = \{(1,0), (8,3), (2,1)\}$. One example of a $\lambda/2$-distinguishing but not $\lambda$-distinguishing is $\Gamma^{\lambda/2} = \{(11, 3), (4, 2), (13, 0)\}$.
  
  \begin{figure}
      \centering
  \begin{tikzpicture}[scale=1.2, every node/.style={transform shape}]
      \centering
      \draw[step=1cm,color=black] (-2,-2) grid (2,2);
      \node at (-1.5,+1.5) {1};
      \node at (1.5,+1.5) {1};
      \node at (-1.5,-1.5) {1};
      \node at (1.5,-1.5) {1};
      \node at (-0.5, 0.5) {$s_1$};
  \end{tikzpicture}
  \caption{A $4 \times 4$ gridworld MDP with start state at $(1,1)$ and reward of 1 in four corners}
  \label{fig:gridworld}
  \end{figure}
  
  \textbf{Performance metric}. At the end of each episode, the two AOMultiRL agents and the one-episode UCBVI agent obtain their estimated model $\hat{P}$. The estimated optimal policy computed based on $\hat{P}$ is run for $H_1 = 200$ steps starting from $(1,1)$. The average per-episode reward (APER) in episode $k = 1, 2, \dots, K$ of an agent is defined as
  
  \begin{align}
      \mathrm{APER}(k) = \frac{\sum_{i=1}^k\sum_{j=1}^{H_1}r_{i, j}}{k}
  \end{align}
  
  where $r_{i,j} = r(s^i_j, a^i_j)$ the reward this agent received in step $j$ of episode $i$.
  
  \textbf{Horizon settings} For AOMultiRL2, the horizons of the clustering phase in two stages are different since the distinguishing sets in the two stages are different. In order to make a fair comparison with other algorithms, the horizon of the learning phase is set to $H_1 = 0$ in stage 1 and $H_1 = 200$ in stage 2. Since we assumed that stage 2 is dominant, the goal of the experiment is to examine whether a $\lambda/2$-distinguishing set can be discovered and how effective that set can be. We observe that AOMultiRL2 is able to discover the same $\lambda/2$-distinguishing set $\{(14,1), (7,2), (13,0)\}$ in all 10 runs. Since this set also has an optimal size of $3$, in stage 2 the clustering phase's horizon $H_0$ of AOMultiRL2 is identical to that of AOMultiRL1.

\end{document}